\documentclass[journal]{IEEEtran}
\usepackage{amsthm}
\usepackage{comment}
\usepackage{amsmath} 
\usepackage{amssymb}  
\usepackage{amsfonts} 
\usepackage{booktabs}
\usepackage{cases}
\usepackage{color}
\usepackage{epstopdf}
\usepackage{graphicx}
\usepackage{multirow}
\usepackage{multicol}
\usepackage{makecell}
\usepackage{subeqnarray}
\usepackage{tabularx}
\usepackage{threeparttable}
\usepackage[switch]{lineno}
\usepackage{epsfig,amssymb,bm,dsfont}
\usepackage{indentfirst}
\usepackage{cuted}
\usepackage{booktabs}
\usepackage{MnSymbol}
\usepackage{pifont}
\newtheorem{theorem}{Theorem}
\newtheorem{lemma}{Lemma}
\newtheorem{proposition}{Proposition}
\newtheorem{remark}{Remark}

\newtheorem{corollary}{Corollary}
\usepackage[caption=false]{subfig}
\ifCLASSINFOpdf
\else
\fi

\usepackage{algorithm}  
\usepackage{algorithmic} 
\begin{document}


\title{
Learning Predictive Control with Deep Koopman Operators for Autonomous Vehicle Motion Planning 
}

\author{
 Xinglong Zhang, Yongqian Xiao, Haotian Cao,  Xing Zhou, Xin Yin, Xin Xu
	
\thanks{\textcolor{black}{This work was supported in part by the National Natural Science Foundation of China under Grants T2521006, 62533021, and U24A20279; in part by Science and Technology Innovation Program of Hunan Province under Grant 2024RC3145.}.
}
}
\markboth{}
{Shell \MakeLowercase{\textit{et al.}}: Bare Demo of IEEEtran.cls for IEEE Journals}
\maketitle

\begin{abstract}
Model Predictive Control (MPC) is widely used for autonomous-vehicle (AV) motion planning, but its real-time applicability is often limited by the need for accurate models and online solution of nonlinear, nonconvex optimization problems in dynamic road environments. Actor--critic reinforcement learning offers a promising alternative for online policy generation, yet its policy-learning process often lacks explicit control-theoretic structure. This article proposes a learning predictive control (LPC) framework with deep Koopman operators for efficient real-time motion planning under nonconvex constraints. To address nonlinear and uncertain vehicle dynamics, a deep-Koopman-based predictor is used to lift the system into an interpretable linear observable space in a data-driven manner. Unlike traditional MPC, which computes open-loop control sequences, the proposed LPC framework yields a closed-loop state-feedback policy within each prediction interval through receding-horizon actor--critic learning. \textcolor{black}{To ensure safety under nonconvex environmental constraints, LPC constructs convex local surrogate representations of obstacles and defines corresponding potential-field functions. These functions and their gradients are directly embedded into the actor–critic structure, enabling efficient, safety-aware policy learning.} Extensive simulations and real-world experiments on the HongQi-EHS3 platform demonstrate favorable performance in diverse obstacle-avoidance scenarios in terms of safety, computational efficiency, and driving comfort, compared with benchmark methods such as CBF-MPC and LMPCC.

\textbf{\textit{Note to Practitioners}}---\textbf{This paper addresses real-time motion planning for autonomous vehicles in dynamic road environments, where conventional MPC can be limited by model dependence and the need to solve nonlinear, nonconvex optimization problems online. The proposed learning predictive control framework combines deep Koopman modeling with receding-horizon actor--critic learning, so that vehicle dynamics can be predicted in a lifted linear space and motion policies can be updated online in closed-loop form. In addition, road-boundary and obstacle information are incorporated directly into the policy-learning process through safety shaping and gradient embedding, rather than being handled only by a separate correction module. For practitioners, this provides a computationally efficient and safety-aware motion-planning scheme suitable for onboard implementation. Simulations and real-vehicle experiments on the HongQi-EHS3 platform show favorable safety, comfort, and real-time performance in representative obstacle-avoidance tasks, including straight-road, curved-road, and mixed-surface scenarios. Current validation focuses on representative operating conditions, and further extension to more diverse dynamic-traffic scenarios and broader road/surface conditions remains an important next step.}

\end{abstract}

\begin{IEEEkeywords}
	\textcolor{black}{Model Predictive Control, deep Koopman operators, actor-critic reinforcement learning, motion planning, autonomous vehicles.}
\end{IEEEkeywords}

\IEEEpeerreviewmaketitle
\section{Introduction} 
\IEEEPARstart
{I}{n} recent years, autonomous driving technology has advanced rapidly in structured and semi-structured environments, such as public roads, smart ports, and unmanned mining fields. Nevertheless, improving the real-time maneuverability of autonomous vehicles (AVs) in dynamic environments remains a major challenge, and stronger environmental adaptability and learning capability are still essential \cite{zhang2022receding, liu2018convex}. A typical autonomous driving system integrates perception, navigation, decision-making, trajectory planning, and tracking control \cite{kuutti2018survey}, among which trajectory planning plays a central role because it directly affects both maneuverability and safety.

Among trajectory planning methods, optimization-based planners aim to generate safe future motions by minimizing a prescribed cost subject to safety-related constraints. A representative example is Model Predictive Control (MPC), which exploits predictive information over a finite horizon while accounting for vehicle dynamics and obstacle avoidance constraints~\cite{chen2019autonomous, zeng2021safety}. However, MPC performance depends strongly on model accuracy and may deteriorate when simplified or linearized vehicle models fail to capture complex dynamics, thereby limiting adaptability across varying conditions~\cite{liu2018convex, brito2019model}. In addition, the resulting nonlinear and non-convex optimization problems are often computationally demanding and susceptible to solver instability. Reinforcement learning (RL), by contrast, offers a data-driven way to learn complex control policies directly, but existing RL methods still face challenges in sample efficiency, interpretability, and safety assurance~\cite{liu2020multi, huang2016lateral}. {Ghalamzan Esfahani, Amir M.Recent studies have advanced online policy generation by integrating reinforcement learning within a receding-horizon framework, enabling efficient online policy updates~\cite{zhang2022robustLPC,dong2018functional,zhang2022receding}. However, designing an online actor–critic learning mechanism that ensures safety under complex environmental constraints remains an open challenge, as it requires balancing learning efficiency, adaptability, and safety.} 

In motion planning under unknown environments, a high-fidelity yet tractable dynamic model is essential for achieving safe and efficient control under constraints. In this context, recent advances in Koopman operator theory provide a promising avenue for reducing both modeling and computational complexity~\cite{korda2018linear,ling2020koopman,tian2025physically}. By lifting nonlinear dynamics into a higher-dimensional observable space, the Koopman operator yields an approximately linear system representation, thereby facilitating optimization-based control with improved computational efficiency. {\color{black}This perspective motivates our policy learning strategy in the linear Koopman observable space, enabling efficient and reliable online policy updates.} 

\textcolor{black}{In this article, we propose a computationally efficient learning predictive control (LPC) framework that leverages deep Koopman operators and a receding-horizon reinforcement learning mechanism to enable real-time, safety-aware motion planning for autonomous vehicles.
The main contributions are summarized as follows.}

\begin{itemize}
    \item \textcolor{black}{We develop a unified LPC framework for autonomous vehicle motion planning under unknown dynamics and non-convex environmental constraints. The framework leverages deep Koopman operators to lift the unknown vehicle dynamics into a linear observable space. Within each prediction interval, a finite-horizon actor–critic learning mechanism is formulated in this space, enabling fast and reliable closed-loop state-feedback policy generation, rather than computing open-loop control sequences online as in conventional numerical MPC. This design facilitates successive and reliable online policy improvement with real-time implementation capability.}

    \item \textcolor{black}{Unlike existing safe RL approaches that rely on online safety filters~\cite{kong2025differential,xie2025cbf}, we propose a unified safety-aware policy learning method that directly incorporates the potential function and its gradient as basis functions in both the actor and critic. The associated weights are updated online with low computational overhead, enabling efficient generation of repulsive forces for obstacle avoidance. 
    }

    \item \textcolor{black}{We validate the proposed framework through extensive simulations and real-world experiments on the HongQi-EHS3 platform. The results demonstrate favorable performance in safety, driving comfort, and computational efficiency relative to benchmark methods such as Control Barrier Function-based MPC (CBF-MPC) \cite{zeng2021safety} and Local Model Predictive Contour Control (LMPCC) \cite{brito2019model}, supporting the practical applicability of the proposed method.}
\end{itemize}

The rest of the article is organized as follows: Section II reviews the related literature. Section III presents the deep-Koopman-based vehicle dynamics modeling and the handling of road-boundary and obstacle constraints. Section IV details the proposed LPC framework with deep Koopman vehicle models and actor--critic learning. Section V reports the simulation and real-vehicle results. Section VI concludes the article and discusses future work.

\section{Related Work}
\textbf{MPC-based trajectory planning}: MPC is a widely used model-based approach that optimizes control objectives over a future horizon and has been extensively applied to robot motion planning and control \cite{nubert2020safe, lazcano2021mpc}. In AV motion planning, obstacles and road boundaries are often incorporated into the optimization problem as safety constraints or potential fields \cite{zhang2022receding, rasekhipour2016potential, gutjahr2016lateral, ji2016path}, which typically results in a nonlinear optimization problem to be solved online. Solving this problem yields executable control sequences or safe trajectories while balancing performance and safety. Although advanced solvers such as IPOPT and SNOPT can handle nonlinear programs, they often suffer from high computational cost and occasional numerical failures. Linearization-based methods can simplify the problem into a QP form~\cite{chen2019autonomous}, improving real-time reliability, but at the expense of model fidelity, thereby creating a trade-off between computational efficiency and control performance.

\textbf{Data-driven vehicle modeling with Koopman operators}: The Koopman operator has recently emerged as a promising data-driven modeling tool for control-oriented design~\cite{korda2018linear}. By lifting nonlinear dynamics into a higher-dimensional space, it enables an approximately linear representation that facilitates control design and optimization without relying on simplified or locally linearized models. Although the Koopman operator is inherently infinite-dimensional, finite-dimensional approximations such as Extended Dynamic Mode Decomposition (EDMD) have been developed. However, EDMD usually depends on manually designed kernel functions, which has motivated deep-learning-based approaches for automatically constructing observation functions. For example, adaptive basis function learning with deep neural networks has been proposed to improve modeling flexibility \cite{li2017extended}, while a variational deep Koopman framework has been introduced to learn Gaussian-distributed observables, although its use of LSTMs and GRUs increases computational complexity \cite{morton_deep_2019}. In the AV domain, a data-driven Koopman-EDMD framework was developed for modeling vehicle dynamics on a scaled platform for real-time trajectory tracking~\cite{joglekar2023data}. Moreover, a deep Koopman model combined with neural networks and LSTM has been applied to freeway traffic flow prediction and ramp metering~\cite{gu2023deep}. Recent studies have further integrated attention mechanisms such as Transformers \cite{geneva2022transformers}, and combined Koopman operators with convolutional architectures for image-based dynamics modeling \cite{bai2022characterization, Heijden2021IROS}, further demonstrating the versatility of Koopman-based methods. In addition, bilinear vehicle models built on Koopman representations have shown promising performance in control applications such as deployment on resource-constrained systems~\cite{zhang2020kmpc} and lane keeping with obstacle avoidance~\cite{yu2022kmpc}. These studies highlight the potential of Koopman-based methods for dynamic system modeling and control.

\textbf{Data-driven trajectory planning}: Recent advances in intelligent control have also highlighted the potential of data-driven methods for AVs, particularly those based on neural networks \cite{coelho2022review, kiran2021deep}. Supervised learning approaches employ deep networks to approximate nonlinear MPC policies, improving online efficiency but typically requiring extensive offline data and offering limited adaptability after deployment \cite{lee2022real}. In contrast, vision-based end-to-end (E2E) control, usually driven by convolutional neural networks (CNNs), reduces reliance on explicit modeling but still faces difficulties in integrating high-level decision-making, which hinders practical deployment \cite{abolghasemi2020accept, jhung2018end}. Nevertheless, incorporating imitation learning with expert demonstrations and collision data has improved the performance of E2E methods \cite{lee2020mixgail}. Although such methods are promising for tracking control and may reduce dependence on costly mapping systems, their lack of transparency and safety guarantees still limits their practical use.

RL provides an appealing alternative by enabling online policy learning for complex AV tasks~\cite{kiran2021deep}. Emerging studies involving large language models (LLMs) also suggest possible improvements in interpretability \cite{fu2023drive, Cui2023drive}. Within RL, deep RL (DRL) has been widely investigated for motion planning and control in diverse scenarios \cite{chen2019model, wu2022uncertainty, chen2021interpretable}. As another important line of research, Approximate Dynamic Programming (ADP) extends dynamic programming ideas to complex nonlinear systems \cite{wang2013self}. For AV control, ADP methods often adopt actor--critic structures to approximate optimal policies and value functions. For example, Dual Heuristic Programming (DHP) has been used for trajectory tracking through steering kinematics \cite{zhao2014approximate}, while related studies model lateral control as a Markov Decision Process (MDP) and develop DHP-based controllers accordingly \cite{huang2016lateral, huang2017parameterized}. More recently, Receding-Horizon RL (RHRL) approaches have combined actor--critic learning with the receding-horizon mechanism for more efficient policy generation \cite{xu2018learning,zhang2022robustLPC,dong2018functional,zhang2022receding}. Unlike traditional MPC, these methods generate explicit feedback policies rather than open-loop control sequences. Building on this line of research, this work extends the RHRL framework to real-time AV motion planning by incorporating Koopman-based dynamics modeling and actor--critic learning, thereby enabling fast policy optimization in a higher-dimensional observable space under unknown dynamics.

\section{Vehicle Dynamics Modeling and Safety Constraints Handling}
This section frames the challenge of navigating under environmental safety constraints, such as obstacles and road boundaries, for AV motion planning treated as an online optimization problem. To address this, we employ the deep Koopman approach to model vehicle system dynamics, transforming nonlinear dynamics into a linear observation space for efficient control design. 
\textcolor{black}{Additionally, convex local approximations are used to construct surrogate obstacle representations, from which potential field functions are built for soft safety handling within a unified actor--critic framework, thereby allowing safety information to participate directly in online policy learning.}

\subsection{Vehicle Dynamics Modeling}
Traditional model-based control methods for motion planning often rely on a simplified AV bicycle model, involving vehicle parameters such as front/rear cornering stiffness, total mass, and yaw inertia. In this work, however, the vehicle model is represented in discrete time as
\begin{equation}\label{equDDK:linear_dynamics}
x_{k+1} = f\left(x_k,u_k\right),
\end{equation}
\textcolor{black}{where $f(\cdot)$ denotes an unknown nonlinear vehicle dynamics map. } More specifically, the state at time $k$ is defined as $x_k=[p_{x,k},\,p_{y,k},\,\psi_k,\,v_{x,k},\,v_{y,k},\,\omega_k] \in \mathbb{X}$, where $(p_{x,k},p_{y,k})$ denotes the vehicle position in a local coordinate system, $\psi_k$ is the yaw angle, and $(v_{x,k},v_{y,k},\omega_k)$ denote the longitudinal velocity, lateral velocity, and yaw rate, respectively. The input is given by $u_k=[\zeta_k,\eta_k]^\top \in \mathbb{U}$, where $\zeta_k$ is the steering angle and $\eta_k$ is the longitudinal acceleration command.

\textcolor{black}{Since the state includes both lateral and yaw dynamics, Eq.~(\ref{equDDK:linear_dynamics}) can be interpreted as a compact nonlinear dynamic representation of vehicle motion, whose unknown mapping $f(\cdot)$ implicitly captures the combined effects of lateral--yaw coupling, wheel--road interaction, and other nonlinear dynamic characteristics reflected in the data, rather than a purely geometric kinematic relation. However, the corresponding dynamic parameters are difficult to estimate accurately, which can degrade model fidelity in nonlinear operating regimes. To address this issue, we adopt a deep Koopman framework to learn a lifted surrogate model of the nonlinear dynamics (\ref{equDDK:linear_dynamics}). In this lifted observable space, the Koopman model evolves approximately linearly, enabling efficient forward prediction for online motion planning.}

\subsection{Deep Koopman Modeling Approach}
The Koopman operator provides a powerful framework for modeling nonlinear systems by lifting their dynamics into a higher-dimensional space where they evolve linearly, enabling efficient and interpretable control design. 
However, identifying suitable observation functions for complex AV systems is challenging. To address this, we combine the Koopman operator with deep neural networks, leveraging their ability to learn expressive, data-driven observation functions for accurate vehicle dynamics modeling. As shown in Fig.~\ref{fig:DKframe}, the deep Koopman maps the nonlinear system to a linear observation space $\mathcal{O}$, the encoder $\boldsymbol{\phi}(\cdot, \theta_e)$, parameterized by $\theta_e$ and implemented as a Multi-layer Perceptron (MLP), transforms the nonlinear vehicle state space into $\mathcal{O}$, where dynamics are governed by linear matrices $A$ and $B$. Meanwhile, the decoder $\tilde{\boldsymbol{\phi}}(\boldsymbol{\phi}(x_k))$, an MLP parameterized by $\theta_d$, maps the observation state back to the original state space, enabling state predictions in the nonlinear domain. This framework is then expressed as
\begin{equation}\label{equDDK:DDK}
\left\{
\begin{aligned}
\hat{\boldsymbol{\phi}}(x_{k+1}) &= A\,\boldsymbol{\phi}(x_k)+B u_k,\\
\hat{x}_{k+1} &= \tilde{\boldsymbol{\phi}}\!\left(\hat{\boldsymbol{\phi}}(x_{k+1})\right).
\end{aligned}
\right.
\end{equation}
where $A\in\mathbb{R}^{L\times L}$ and $B\in\mathbb{R}^{L\times m}$, with $L$ denoting the dimension of the observation space. $\hat{\boldsymbol{\phi}}(x_{k+1})$ denotes the predicted observation state at time $k+1$, from which the corresponding predicted vehicle state $\hat{x}_{k+1}$ is obtained via the decoder.

\begin{figure}
\centering
\includegraphics[width=0.96\columnwidth]
{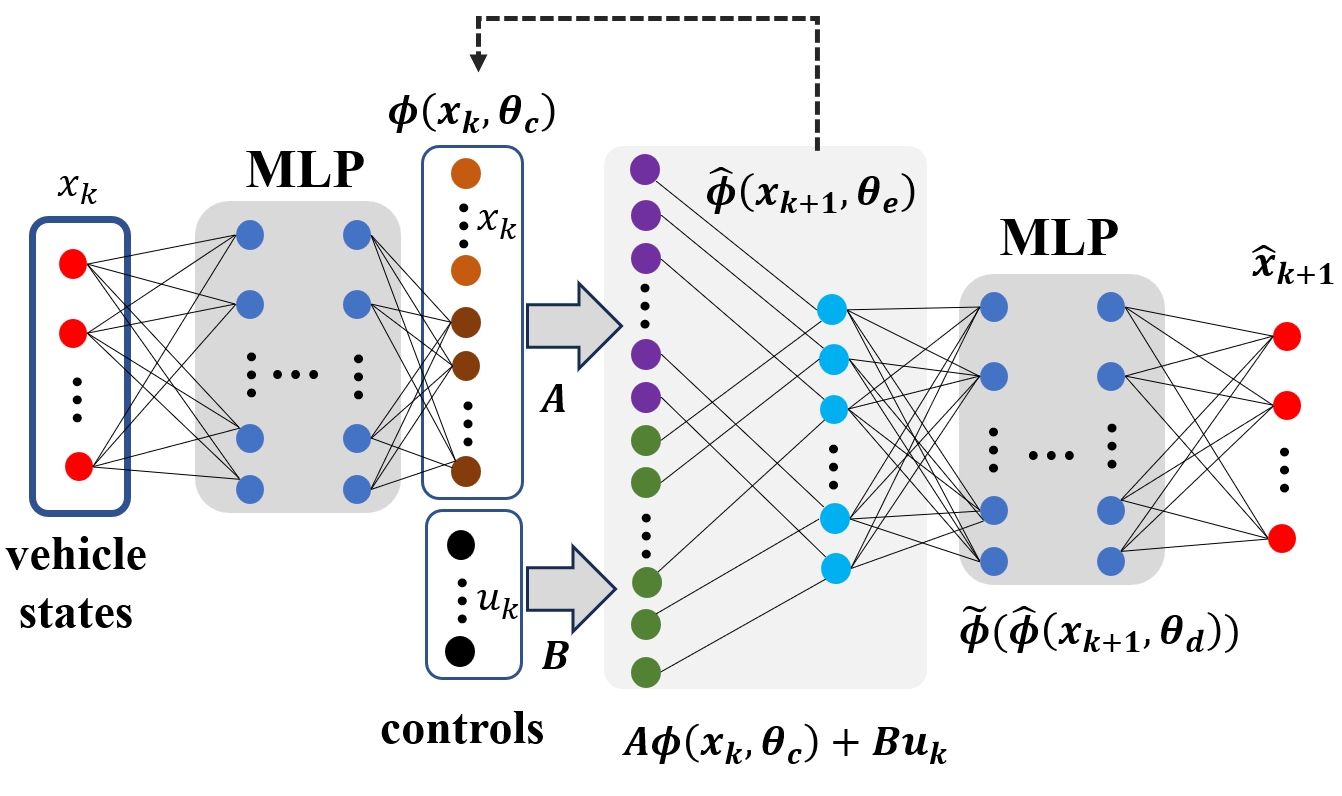}
\caption{Deep Koopman vehicle dynamics modeling framework (MLP: Multi-layer Perception).}
\label{fig:DKframe}
\end{figure}
To approximate vehicle dynamics over an extended time window, minimizing the multi-step prediction error takes precedence over the one-step error. Accordingly, the loss function includes multi-step reconstruction prediction loss and reconstruction loss, as well as multi-step prediction loss within the observation space. Unlike EDMD, which uses one-step prediction approximations solvable analytically, this problem’s complexity precludes an analytical solution due to its multi-step nature. Instead, a data-driven training approach is adopted, in which the inclusion of a multi-step loss function renders least-squares methods impractical, but batch gradient descent provides an effective training alternative. More training details are provided in our previous work \cite{xiao2023deepedmd}.

\subsection{Constraints Handling}
\textcolor{black}{RL excels in adaptive control for AVs but struggles to incorporate hard nonconvex safety constraints directly into policy design and online learning. To obtain a tractable formulation for safety-aware policy learning, we first construct convex local surrogate representations of obstacle boundaries and the associated potential field functions. These are then integrated into the actor–critic architecture to enable safe policy learning. This design yields a unified safety-aware learning framework for motion planning under complex non-convex constraints.}

\textit{(1) Road Boundary Constraints}: The feasible set for road boundary constraints is defined as $\Omega_{k}^b=\{p_k|p_{\min}\le p_k-p_{ref,k}\le p_{\max}\}$, where $p_k=\text{I}_{2\times n}\cdot x_k$ and $p_{ref,k}=\text{I}_{2\times n}\cdot x_{ref,k}$ represent the coordinate vectors of the autonomous vehicle and its nearest reference point, respectively. The terms $p_{\min},p_{\max}\in \mathbb{R}^2$ denote the time-varying road boundaries. To enforce lane constraints on the vehicle, a logarithmic potential field function can be introduced:
\begin{equation}\label{equRHDHP:bound_barrier_function}
	\mathcal{B}_b(d_b) = -\ln \left( \frac{\gamma _1\left( d_b-d_{b,\min} \right)}{\gamma _1\left( d_b-d_{b,\min} \right) +1} \right) -\ln \left( \frac{\gamma _2\left( d_{b,\max}-d_b \right)}{\gamma _2\left( d_{b,\max}-d_b \right) +1} \right),
\end{equation}
where $d_b$ is the normal-directional distance to the road boundary, $d_{b,\min}$ and $d_{b,\max}$ are the minimum and maximum distances, and $\gamma_1, \gamma_2 > 0$ control the function’s suppression effect near boundaries. The function, inactive at safe distances to save computational resources, balances safety and optimality based on $\gamma_1$ and $\gamma_2$.

\textit{(2) Obstacle Avoidance Constraints}: The obstacle avoidance constraint set $\{ p_k \mid p_k \notin \Omega_k^o \}$ is potentially nonconvex, where $\Omega_k^o$ is the obstacle region and $p_k = \bar{C} \boldsymbol{\phi}(x_k)$ is the vehicle’s position in the observation space. We apply the Convex Feasible Set (CFS) method~\cite{liu2018convex} to approximate this set with a convex representation, defining the distance from the obstacle as $d(p, \Omega_k^o) := \min_{z \in \Omega_k^o} d(p, z)$, where $d(p, z)$ is the Euclidean distance. The vehicle is approximated by $j$ circular envelopes, yielding the constraint:
\begin{equation}
	\bar{C}\boldsymbol{\phi }\left( x_k \right) \in \mathbb{F} _k=\cap _{i=1}^{j}\mathbb{F} _{k}^{i},
\end{equation}
where $\mathbb{F} _{k}^{i}=\left\{ \tilde{p}_{k,i}~|~d( \tilde{p}_{k,i},\Omega _{k}^{o} ) \ge d_{o,\min} \right\}$, and $\tilde{p}_{k,i}$ is a point outside the $i$-th envelope, and $d_{o,\min}$ is the safety distance threshold. The standard CFS method produces time-varying polyhedral constraints that are prone to instability due to sensitivity to initial conditions, often requiring additional guidance for effective obstacle avoidance. To address these limitations, an enhanced CFS approach is proposed, incorporating parameters $l_f$ and $l_b$ to reshape polyhedral edges for better safety and comfort (Fig.~\ref{figSRCKNet:constraint_bend_road}). The parameters $h_s$ and $w_s$ define the expanded dimensions of the obstacle, considering the positioning and control inaccuracies. 
A convex polygon, formed by points $P_b$ and $P_f$ (rear and front corners), extends the obstacle’s bounding box by $l_b$ backward and $l_f$ forward along its longitudinal axis from the lane centerline, guided by the minimal safe following gap. It encloses the obstacle’s circular envelopes, ensuring a convex representation.
The resulting linear safety constraint is:
\begin{equation}\label{equRHDHP:obs_linear_cons}
	G_i(\boldsymbol{\phi}(x_k)) = \boldsymbol{L}_{i,k} \boldsymbol{\phi}(x_k)+d_{i,k} - g_{i,k},
\end{equation}
where $\boldsymbol{L}_{i,k} = [b_{i,k}, c_{i,k}, \boldsymbol{0}_{1\times L-2}]$, $b_{i,k}$, $c_{i,k}$ are the coefficients of the normal vector to the edge, $d_{i,k}$ is the offset determined by the position of the edge relative to the origin in the coordinate system, while $g_{i,k} \leq 0$ is the safety margin shrinking the obstacle avoidance constraint for safety under dynamics modeling errors. Then the obstacle potential field function as a soft constraint is defined by
\begin{equation}\label{equRHDHP:obs_barrier_function}
	\mathcal{B} _o\left( \boldsymbol{\phi} \left( x_k \right) \right) =\sum_{i=1}^{m_o}{\frac{1}{ G_i\left( \boldsymbol{\phi} \left( x_k \right) \right) ^2+\epsilon _{\boldsymbol{\phi },i}}},
\end{equation}
where $m_o$ is the number of edges of the convex polygon representing the obstacle. 
\textcolor{black}{The above potential-field terms act as soft safety penalties, while their associated gradients serve as repulsive forces within the actor–critic structure for online safe learning. In implementation, the regularization term $\varepsilon_{\phi,i}>0$ is introduced to avoid singularity and improve numerical stability; smaller values lead to stronger repulsive effects near unsafe regions and therefore enlarge empirical safety margins.}

\begin{figure}
\centering
\includegraphics[width=0.8\columnwidth]{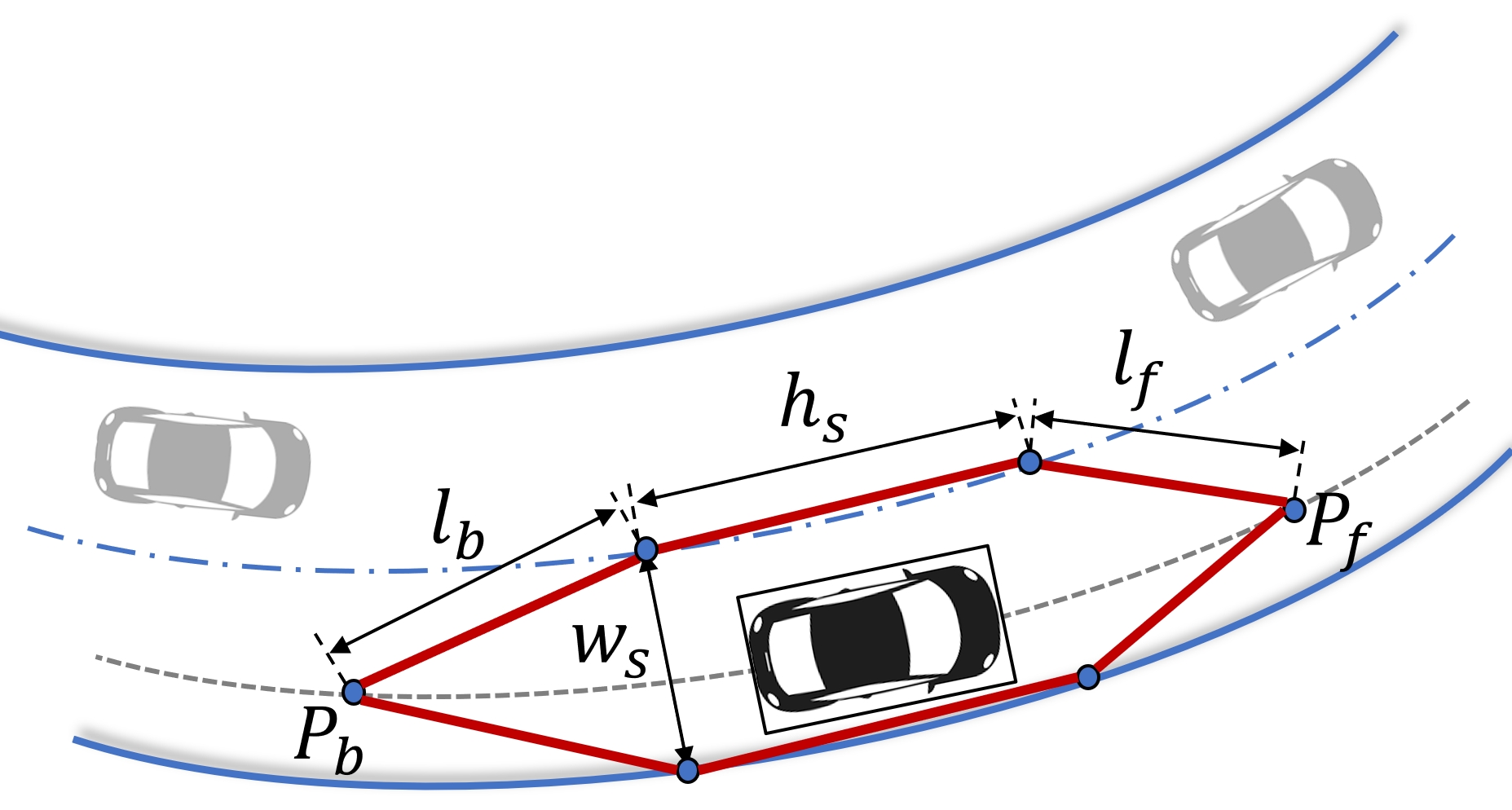}
\caption{Illustration of convex representation for obstacle constraints.}
\label{figSRCKNet:constraint_bend_road}
\end{figure}

\section{LPC framework with Deep Koopman Operators}
This section introduces the LPC framework with a deep Koopman vehicle dynamics model and actor-critic learning. Unlike traditional MPC, which computes an open-loop control sequence, the proposed framework leverages an actor-critic learning strategy to derive a closed-loop control policy within the prediction horizon. Utilizing the deep Koopman vehicle dynamics model as its prediction model, it produces a control policy that seamlessly integrates planning and control for direct vehicle application. Using the deep Koopman model for predictions, it reformulates the value function with obstacle functions, derives optimal control via the Hamilton-Jacobi-Bellman (HJB) equation, and implements online learning of the control policy and value function in the actor-critic framework. 
\subsection{Value Function Reconstruction and Optimal Control Strategy}
Since the RL approach relies on trial-and-error exploration and cannot directly enforce hard constraints, we address this by integrating obstacle functions, representing safety constraints such as road boundaries and obstacle avoidance, into the performance index. This transforms the original constrained optimization problem into an unconstrained one, leveraging soft constraints for simpler and safer resolution. Let $\boldsymbol{\phi}_{e}=\boldsymbol{\hat{\phi}}-\boldsymbol{\phi}_{\mathrm{ref}}$,  where $\boldsymbol{\phi}_{\text{ref}}(x_{k})$ is the reference trajectory’s observed state. The reward function is defined as
\begin{equation}
	\begin{aligned}
		r(\boldsymbol{\phi}_e(x_{k}), u_{k}) = &\left\|\boldsymbol{\phi}_e(x_{k})\right\|_{Q}^{2}+\|u_{k}\|_{R}^{2} \\
		& + \gamma_{b} \mathcal{B}_b(d_{b,k}) + \gamma_{o} \mathcal{B}_o(\boldsymbol{\phi}(x_{k})),
	\end{aligned}	
\end{equation}
where $Q \in \mathbb{R}^{L \times L}$ and $R \in \mathbb{R}^{m \times m}$ are positive definite matrices penalizing deviations in observed states and control inputs, respectively, $\gamma_b > 0$ and $\gamma_o > 0$ are penalty coefficients weighting the soft constraints. 

Similarly, the terminal cost $V_f(\boldsymbol{\phi}_e(x_{k+N_p}))$ is reformulated as
\begin{equation}
	\begin{aligned}
		V_f(\boldsymbol{\phi}_e(x_{k+N_p})) =& \left\|\boldsymbol{\phi}_e(x_{k+N_p \mid k})\right\|_{P}^{2}+\gamma_{b}\mathcal{B}_b(d_{b,k+N_p\mid k}) \\
		& 
		+ \gamma_{o} \mathcal{B}_o(\boldsymbol{\hat{\phi}}(x_{k+N_p\mid k})),
	\end{aligned}
\end{equation}
where $P\in\mathbb{R}^{L\times L}$ is a positive definite matrix. 

After the value function reconstruction process, building on the deep Koopman vehicle dynamics model and safe constraint handling, we formulate the AV motion planning problem as a finite-horizon optimization at each discrete time step $k$:
\begin{equation}\label{equRHDHP:planning_optimization1}
	\begin{aligned}
		\min_{u_{k:k+N_p-1}} V_k &= \sum_{i=0}^{N_p-1} r(\boldsymbol{\phi}_e(x_{k+i\mid k}),u_{k+i\mid k})+ V_f(\boldsymbol{\phi}_e(x_{k+N_p}))\\
		\text{s.t.} \,\,\, &\boldsymbol{\hat{\phi}}(x_{k+1}) = A\boldsymbol{\phi}(x_k) + B u_k
	\end{aligned}
\end{equation}
where $N_p$ denotes the prediction horizon length. Based on the Bellman optimality principle, the HJB equation in the prediction horizon $\tau\in[k,k+N_p-1]$ can be written as
\begin{equation}\label{equRHDHP:optimal_value}
	V_{\tau}^{*}=\min_{u_{\tau}\in \mathbb{U}} \,\,(r\left( \boldsymbol{\phi }_{e,\tau} ,u_{\tau} \right) +V_{\tau+1}^{*}),
\end{equation}
where $\tau \in [k,k+N_p-1]$, $\boldsymbol{\phi}_{e,\tau}=\boldsymbol{\phi}_e(x_{\tau})$ denotes the observed states, and $V_{\tau}^*$ denotes the optimal value function within the prediction horizon. 

Thus, the optimal control $u_{\tau}^{*}$ is given by
\begin{equation}\label{equRHDHP:optimal_control}
	u_{\tau}^{*}=\mathop {\mathrm{arg}\min} \limits_{u_{\tau}\in \mathbb{U}}\,\,(r\left( \boldsymbol{\phi }_{e,\tau} ,u_{\tau} \right) +V_{\tau+1}^{*}).
\end{equation}


Moreover, the costate value function is defined as  $\lambda_{\tau} =\frac{\partial V_{\tau}}{\partial \boldsymbol{\phi }_{e,\tau}}$, 
and noticing $\lambda^*_{\tau+1}=\frac{\partial V^*_{\tau+1}}{\partial \boldsymbol{\phi}_{e,\tau+1}}$, the optimal costate value function $\lambda^*$ can be expressed as
%
\begin{equation}\label{equRHDHP:optimal_lambda}
	\lambda _{\tau}^{*}=\frac{\partial V_{\tau}^{*}}{\partial \boldsymbol{\phi}_{e,\tau}}=2Q\boldsymbol{\phi }_{e,\tau}+A^{\top}\lambda_{\tau+1}^{*}+\gamma _{b}\bm{\beta}_{b,\tau}+\gamma _{o}\bm{\beta}_{o,\tau},
\end{equation}
and
\[\bm{\beta}_{b,\tau} =\frac{\partial \mathcal{B} _b\left( d_{b,\tau} \right)}{\partial d_{b,\tau}}\frac{\partial d_{b,\tau}}{\partial \boldsymbol{\phi }_{e,\tau}}, \quad \bm{\beta}_{o,\tau} =\frac{\partial \mathcal{B}_o\left( \boldsymbol{\hat{\phi}}_{\tau} \right)}{\partial \boldsymbol{\phi }_{e,\tau}}, \]
and 
$A=\frac{\partial \boldsymbol{\phi }_{e,\tau+1}}{\partial \boldsymbol{\phi}_{e,\tau}}$ 
could be obtained from \eqref{equDDK:DDK}, $d_{b,{\tau}}$ represents the Euclidean distance between the autonomous vehicle and the nearest point on the road boundary.

Then, the optimal terminal costate value function is determined by
\begin{equation}
	\lambda_{k+N_p}^{*} = 2P\boldsymbol{\phi}_{e,k+N_p} + \gamma_{b} \bm{\beta}_{b,k+N_p} + \gamma_{o} \bm{\beta}_{o,k+N_p},
\end{equation}
where 
\[\bm{\beta}_{b,k+N_p} = \frac{\partial \mathcal{B}_b(d_{b,k+N_p})}{\partial d_{b,k+N_p}} \frac{\partial d_{b,k+N_p}}{\partial \boldsymbol{\phi}_{e,k+N_p}}, \quad \bm{\beta}_{o,k+N_p} = \frac{\partial \mathcal{B}_o(\boldsymbol{\hat{\phi}}_{k+N_p})}{\partial \boldsymbol{\phi}_{e,k+N_p}}. \] 

\subsection{Safe Policy Learning in the Prediction Interval}
Traditional infinite-horizon ADP methods for online control policy learning lack efficiency. To address this, we propose the LPC framework with a deep Koopman vehicle dynamics model. This approach converts infinite-horizon optimization into finite-horizon problems, enabling iterative policy learning within the prediction horizon $[k, k+N_p-1]$ to ensure convergence. The deep Koopman model’s linear dynamics in the observation space provide high forward inference efficiency, thereby enhancing real-time performance. Furthermore, road boundary and obstacle constraint handling ensure safe and effective avoidance maneuvers. As shown in Fig.~\ref{fig:SLPC}, the proposed framework uses onboard sensor data for the vehicle’s state $x_k$ and obstacle positions. Within the prediction horizon, the deep Koopman model predicts the vehicle’s trajectory and obstacle interactions, while an actor-critic structure learns the control policy and value function. The resulting control sequence $u_{k:k+N_p-1}$ is computed, but only the first element is applied at time $k$. Meanwhile, the optimization repeats at $k+1$ to update the solution.
\begin{figure}[htb]
	\begin{center}
		\includegraphics[width=0.96\columnwidth]{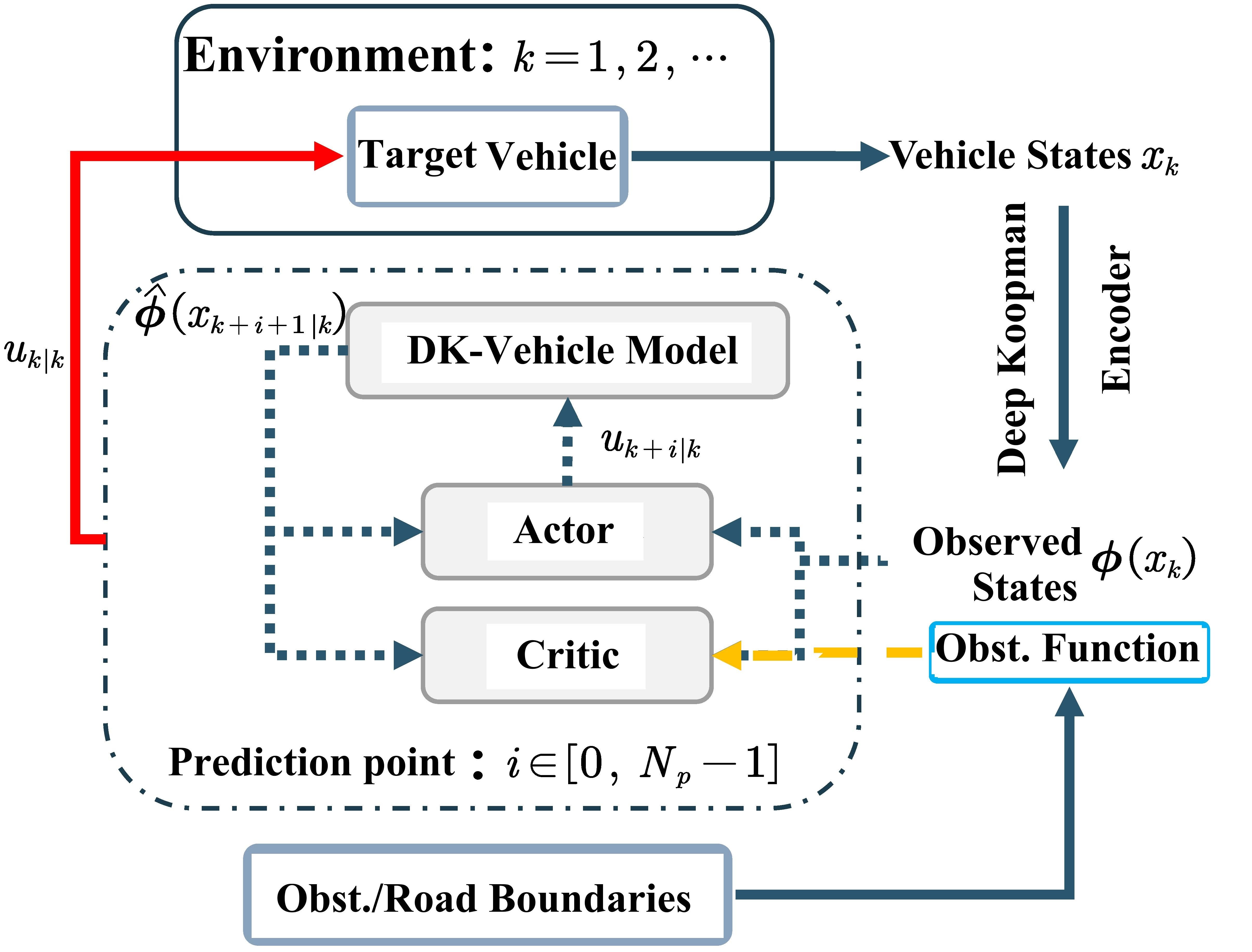}
		\caption{LPC with deep Koopman models and actor-critic learning (Obst.: Obstacle, DK: deep Koopman).}
		\label{fig:SLPC}
	\end{center}
\end{figure}

The LPC framework based on actor-critic structure is used to learn and approximate the optimal value function \eqref{equRHDHP:optimal_value} and the optimal control policy \eqref{equRHDHP:optimal_control}.  Let $V_{\tau}^{(0)}=0$, the proposed method iteratively updates the control policy and function values through policy evaluation and policy optimization within the prediction horizon, respectively.

(1) Policy evaluation:
\begin{subequations}\label{equRHDHP:value_iteration}
    \begin{align}
	V_{\tau}^{(i+1)} &= r\left( \boldsymbol{\phi}_{e,\tau} ,u_{\tau} \right) +V_{\tau+1}^{(i)},\\
	V_{k+N_p}^{(i+1)} &= V_f(\boldsymbol{\phi}_{e,k+N_p}).
    \end{align}	

(2) Policy optimization:
    \begin{align}\label{equRHDHP:policy_improvement}
     u_{\tau}^{(i)} = \mathop {\mathrm{arg}\min} \limits_{u_{\tau}\in \mathbb{U}}\,\,(r\left(\boldsymbol{\phi}_{e,\tau},u_{\tau} \right) +V_{\tau+1}^{(i)}), 
    \end{align}
\end{subequations}
where $\tau \in [k,k+N_p-1]$, and the upper script ${}^{(i)}$ denotes the iteration number.

Next, we present the receding horizon value iteration algorithm based on the costate value function. Let $\lambda_{\tau}^{(0)}=0$, and the policy evaluation and policy optimization iterations within the prediction horizon are given by
\begin{equation}
\begin{aligned}\label{equRHDHP:value_iteration1}
&\lambda_{\tau}^{(i+1)}=2Q\boldsymbol{\phi}_{e,\tau}+\Gamma A^{\top}\lambda_{\tau+1}^{(i)}+ \gamma_{b} \bm{\beta}_{b,\tau} +\gamma_{o} \bm{\beta}_{o,\tau}, \\
&\lambda_{k+N_p}^{(i+1)} = 2P\boldsymbol{\phi}_{e,k+N_p} + \gamma_{b} \bm{\beta}_{b,k+N_p} +\gamma_{o} \bm{\beta}_{o,k+N_p},
\end{aligned}
\end{equation}
\begin{equation}\label{equRHDHP:policy_improvement1}
	u_{\tau}^{(i)}(j)=\begin{cases}
		\tilde{u}_{\tau}^{(i)}(j)&		u_{\min}(j)\le \tilde{u}_{\tau}^{(i)}(j)\le u_{\max}(j),\\
		u_{\min}(j)&		\tilde{u}_{\tau}^{(i)}(j)<u_{\min}(j),\\
		u_{\max}(j)&		\tilde{u}_{\tau}^{(i)}(j)>u_{\max}(j),\\
	\end{cases}
\end{equation}
Here, $\Gamma \in (0,1]$ is the actor--critic update discount factor, which weights the propagation of future co-state information within the prediction horizon, and 
\[ \tilde{u}_{\tau}^{(i)}=-\frac{1}{2} R^{-1} \Gamma B^{\top}\lambda _{\tau +1}^{(i)}, \quad \tau\in [k, k+N_p-1]\] 
and $\tilde{u}_{\tau}^{(i)}(j)$ represents the $j$-th element of $\tilde{u}_{\tau}^{(i)}$. Similarly, $u_{\min}(j)$, $u_{\max}(j)$, $u_{\tau}^{(i)}(j)$ represent the $j$-th elements of $u_{\min}$, $u_{\max}$, $u_{\tau}^{(i)}$, respectively.

The convergence analysis of the finite-horizon policy-learning procedure is deferred to Appendix \ref{app:inner_loop_convergence}.

\subsection{Online Safe Actor-critic Implementation}
This section uses actor-critic learning to implement the proposed approach, where both the actor and critic are realized using neural networks. The actor network is responsible for learning and approximating the optimal control policy, while the critic network evaluates the learned control policy. 
To achieve safe vehicle motion control in complex environments, the co-state value function includes obstacle function terms for road boundary and obstacle avoidance. Therefore, the output layer structure of the critic network is designed as
\begin{equation}\label{equRHDHP:critic_output1}
    \hat{\lambda}_{\tau} = W_{c,\tau}^{\top} h_c(\boldsymbol{\phi}_{e,\tau}) + b_{c,\tau} + \gamma _{b}\bm{\beta}_{b,\tau}+\gamma _{o}\bm{\beta}_{o,\tau},
\end{equation}
where $W_{c,\tau}$ and $b_{c,\tau}$ represent the weight matrix and bias vector of the output layer, respectively. $h_{c}(\boldsymbol{\phi}_{e,\tau})$ represents the output of the last hidden layer. When the autonomous vehicle approaches the road and obstacle constraint boundaries, the last two terms in the equation will increase sharply, causing the vehicle to move away from the road boundary and the obstacle boundary. From equation \eqref{equRHDHP:value_iteration1}, the target co-state value function of the critic network within the prediction horizon $[k, k+N_p]$ is calculated as follows
\begin{equation}\label{equRHDHP:target_value}
\begin{aligned}
    &\bar{\lambda} _{\tau}=2Q\boldsymbol{\phi }_{e,\tau}+\Gamma A^{\top}\hat{\lambda}_{\tau+1}+\gamma _{b}\bm{\beta}_{b,\tau}+\gamma _{o}\bm{\beta}_{o,\tau}, \\
        &\bar{\lambda}_{k+N_p} =2P\boldsymbol{\phi}_{e,k+N_p} + \gamma_{b} \bm{\beta}_{b,k+N_p} + \gamma_{o} \bm{\beta}_{o,k+N_p},
\end{aligned}
\end{equation}
where $\tau \in [k, k+N_p-1]$, and $\hat{\lambda}_{\tau+1}=\hat{\lambda}(\hat{\boldsymbol{\phi}}_{e,\tau+1})$. 

During the algorithm training process, the critic network minimizes the Temporal Difference (TD) error, so the corresponding optimization objective function $\delta _{c,\tau}$ can be designed as follows
\begin{equation}\label{equRHDHP:critic_loss}
	\delta _{c,\tau}=\frac{1}{2}\left\| \hat{\lambda}_{\tau}-\bar{\lambda}_{\tau} \right\| _{2}^{2},
\end{equation}
where $\hat{\lambda}_{\tau}$ is obtained from the critic network, and $\bar{\lambda}_{\tau}$ is calculated using equation \eqref{equRHDHP:target_value}. The update rule for the critic network weights is given by
\begin{equation}\label{equRHDHP:update_critic}
	\theta_{c,{\tau}+1} = \theta_{c,{\tau}} - \alpha_{c,{\tau}} \delta_{c,\tau} \frac{\partial \hat{\lambda}(\boldsymbol{\phi}_{e,\tau})}{\partial \theta_{c,{\tau}}},
\end{equation}
where $\alpha_{c,{\tau}} \in (0, 1]$ is the learning rate for the critic; $\theta_c$ represents the weights of the critic network, which include the weight matrices and bias vectors for each layer of the neural network.

The actor network is implemented using a neural network with the activation function of the final layer being the hyperbolic tangent function, namely,
\begin{equation}\label{equRHDHP:actor_output1}
    \hat{u}_{\tau} = \tanh \left( W_{a,\tau}^{\top} h_{a,\tau} \left( \boldsymbol{\phi}_{e,\tau} \right) + b_{a,\tau} + K_{b,\tau}\bm{\beta}_{b,\tau} + K_{o,\tau} \bm{\beta}_{o,\tau} \right),
\end{equation}
where $W_{a,\tau}$, $b_{a,\tau}$, $K_{o,\tau} \in \mathbb{R}^{m \times L}$, and $K_{b,\tau} \in \mathbb{R}^{m \times L}$ are the weights for the output layer of the actor-network; $h_{a,\tau}(\boldsymbol{\phi }_{e,\tau})$ represents the output of the last hidden layer of the actor-network. 

From Eq. \eqref{equRHDHP:policy_improvement1}, the target policy for the actor-network is given by
\begin{equation}\label{equRHDHP:target_control}
	\bar{u}_{\tau}(j)=\begin{cases}
		\hat{\tilde{u}}_{\tau}(j)&		u_{\min}(j)\le \hat{\tilde{u}}_{\tau}\le u_{\max}(j),\\
		u_{\min}(j)&		\hat{\tilde{u}}_{\tau}(j)<u_{\min}(j),\\
		u_{\max}(j)&		\hat{\tilde{u}}_{\tau}(j)>u_{\max}(j),\\
	\end{cases}
\end{equation}
where $\hat{\tilde{u}}_{\tau} = -\frac{1}{2} R^{-1} \Gamma B^{\top}\hat{\lambda} _{\tau +1}$; $u(j)$ represents the $j$-th element of $u$. 

The optimization objective function for the actor-network is given by
\begin{equation}\label{equRHDHP:actor_loss}
	\delta_{a,\tau} = \frac{1}{2}\left\|\hat{u}_{\tau} - \bar{u}_{\tau}\right\|_2^2,
\end{equation}
where $\hat{u}_{\tau}$ is obtained from equation \eqref{equRHDHP:actor_output1}. The weight update rule for the actor-network is
\begin{equation}\label{equRHDHP:update_actor}
	\theta_{a,\tau+1} = \theta_{a,\tau} - \alpha_{a,\tau} \delta_{a,\tau} \frac{\partial \hat{u}_{\tau}}{\partial \theta_{a,\tau}},
\end{equation}
where $\alpha_{a,\tau}$ is the learning rate for the actor network.

\begin{algorithm}[ht]
	\caption{LPC with deep Koopman Operators}
	\label{algRHDHP:DDK_RHRLSC}
	\begin{algorithmic}[1]
	    \REQUIRE deep Koopman model Eq.(\ref{equDDK:DDK}), hyper-parameters ($Q, R, \Gamma, N_p, \alpha_a, \alpha_c, i_{\max}$), reference trajectory $\mathcal{T}_{\text{ref}}$, safety distances ($d_s, d_p$). Initialize critic weights $\theta_c$ s.t. $\lambda_\tau^0 = 0$, actor weights $\theta_a$. Set $k = 1$; $\text{Flag}_{\text{road}} = 1$ if vehicle is within road safety range, else $0$;
        \WHILE {$\text{Flag}_{\text{road}}==1$ \textbf{and} $d_{o,k}>d_s$}
		\STATE Obtain the vehicle's current state $x_k$ and all obstacle states via onboard sensors; 
		\STATE Convert the vehicle, obstacle states, reference trajectory, and road boundary to local coordinates; 
		\STATE Obtain the current $\boldsymbol{\phi}_{\tau}$ from the deep Koopman vehicle dynamics model \eqref{equDDK:DDK} 
		\STATE $i = 0$; 
		\WHILE {$i < i_{\max}$}
		\FOR{$\tau \in [k, k+N_p-1]$}
		\STATE $\beta_{o,\tau} = \boldsymbol{0}_{L\times 1}$, $\beta_{b,\tau} = \boldsymbol{0}_{L\times 1}$; \hfill $O(L)$
		\IF{$d_{o,k}<d_p$}
		\STATE Calculate $\beta_{o,\tau}$ based on Eq. \eqref{equRHDHP:obs_barrier_function}; \hfill $O(m_o L)$
		\ENDIF
		\IF {$d_{b,k}<d_s$}
		\STATE Compute $\beta_{b,\tau}$ based on Eq. \eqref{equRHDHP:bound_barrier_function}; \hfill $O(L)$
		\ENDIF
		\STATE Obtain $\hat{u}_{\tau}$ from Eq. \eqref{equRHDHP:actor_output1}; \hfill $O(N_{nn})$
		\STATE Input $\hat{u}_{\tau}$ into the deep Koopman model \eqref{equDDK:DDK} to predict the vehicle's next state $\hat{\boldsymbol{\phi}}_{\tau+1}$; \hfill $O(L^2 + Lm)$
		\STATE Calculate $\hat{\lambda}_{\tau+1}$ using Eq. \eqref{equRHDHP:critic_output1}; \hfill $O(N_{nn})$
		\STATE Compute the target co-state value function $\bar{\lambda}_{\tau}$ using Eq. \eqref{equRHDHP:target_value}; \hfill $O(L^2 + m_o L)$
		\STATE Calculate the co-state value function error via Eq. \eqref{equRHDHP:critic_loss} and update critic network weights to $\theta_{c,\tau+1}$ using Eq. \eqref{equRHDHP:update_critic}; \hfill $O(N_{nn})$
		\STATE Compute $\bar{u}_{\tau}$ using Eq. \eqref{equRHDHP:target_control}; \hfill $O(L m)$
		\STATE Compute the actor loss via Eq. \eqref{equRHDHP:actor_loss} and update actor weights to $\theta_{a,\tau+1}$ using Eq. \eqref{equRHDHP:update_actor}; \hfill $O(N_{nn})$
		\ENDFOR
		\STATE $i = i + 1$; 
		\ENDWHILE
		\STATE Apply the first control input from the predicted time horizon to the vehicle system; 
		\STATE Set $k = k + 1$; 
		\ENDWHILE
	\end{algorithmic}
    \noindent\textit{Complexity notation:} $N_{nn}$ denotes the total per-step neural-network operation count associated with the actor--critic implementation; for fixed network architectures, it is proportional to the overall parameter scale of the actor and critic networks.
\end{algorithm}

\begin{remark}
The pseudocode of the proposed method is summarized in Algorithm~\ref{algRHDHP:DDK_RHRLSC}. The per-cycle computational cost is dominated by the inner loop, which performs $i_{\max}$ iterations over a prediction horizon of length $N_p$. For each step in the prediction horizon, the main computations include obstacle-function gradient evaluation, with complexity $O(m_oL)$, lifted-state prediction and costate update based on matrix operations, with complexity $O(L^2)$, and actor--critic network evaluations and updates, whose aggregated cost is denoted by $O(N_{nn})$. Therefore, the overall computational complexity per control cycle is $O\!\left(i_{\max}N_p\left(L^2+m_oL+N_{nn}\right)\right)$. This complexity is polynomial in the main design dimensions, indicating favorable computational scalability and supporting practical implementation. 
\end{remark}

\begin{remark}
{\color{black}The convergence result within each prediction interval, together with the safety property of the predicted trajectory induced by the designed potential-function mechanism, is established but omitted here for brevity. Interested readers are referred to Appendix A for detailed theoretical results.}
\end{remark}

\section{Simulations and Real-world Experiments}

This section evaluates the proposed LPC with Deep Koopman Operators (referred to as “Ours”) through simulations and real-world experiments, comparing its performance against two benchmarks of CBF-MPC and LMPCC. Simulations covered four diverse scenarios involving static and dynamic obstacle avoidance on straight and curved roads, while real-world tests were conducted on the HongQi-EHS3 platform. All simulations are performed on a Windows 11 platform with an AMD Ryzen 9 5950X @ 3.2 GHz and 32GB RAM, with a sampling interval of $t_s = 0.02$ s and prediction horizon $N_p = 30$. 
The parameters of the road boundary and obstacle potential functions were $\gamma_1 =\gamma_2 =1.5,~d_{b,\text{max}}=4,~d_{b,\text{min}}=-4$, $\gamma_o = 9.5$, $\epsilon_\phi = 0.04$, $g = 0.3$; the weight matrices are set as $Q = \text{diag}\{20, 300, 300, 300, 20, 20, 100 \cdot \mathbf{1}_{1 \times (L-n)}\}$, and $R = \text{diag}\{5000, 5000\}$; the learning rates of actor and critic were 0.1 and 0.2, respectively, which were selected based on empirical results from simulation trials. The actor--critic update discount factor is set to $\Gamma=1$ unless otherwise stated.

Besides, \textcolor{black}{during the real-world experiments with the HongQi E-HS3 platform, the vehicle state used by the proposed framework was obtained from fused onboard sensing/localization outputs, including position, heading, longitudinal velocity, lateral velocity, and yaw rate. Since the HongQi E-HS3 is an electric vehicle, the longitudinal control input in the proposed framework was defined as commanded longitudinal acceleration rather than raw pedal percentage. In the low-level implementation, this commanded longitudinal acceleration was tracked by a PID-based longitudinal controller, while the steering command is sent to the steering actuation layer accordingly.}



\subsection{Open-Loop Multi-Step Prediction Verification of the Deep Koopman Model}

\textcolor{black}{We first verify the open-loop multi-step prediction performance of the learned deep Koopman model using both high-fidelity simulation data and real-vehicle data. The high-fidelity CarSim dataset contains 38 trajectories (approximately $8.6\times10^5$ samples) with a sampling interval of 10~ms, where the reference speed was scheduled according to road curvature and reaches up to 60~km/h in low-curvature segments. The real-vehicle dataset was collected from the HongQi E-HS3 platform using fused GNSS/INS-based localization and contains approximately $8.1\times10^5$ samples.}

\textcolor{black}{To quantitatively evaluate the open-loop multi-step prediction performance in the observable space, the mean absolute error (MAE) of the $i$-th observable dimension is defined as}
\begin{equation}
\label{eq:koopman_obs_mae}
E_i=\frac{1}{M_NN_p}\sum_{j=1}^{M_N}\sum_{k=1}^{N_p}
\left|\phi_i\!\left(x_k^{[j]}\right)-\hat{\phi}_i\!\left(x_k^{[j]}\right)\right|,
\end{equation}
\textcolor{black}{where $E_i$ denotes the MAE of the $i$-th observable dimension, $M_N$ is the number of test trajectories, $N_p$ is the prediction horizon, $x_k^{[j]}$ denotes the $k$-th state of the $j$-th test trajectory, and $\phi_i(\cdot)$ and $\hat{\phi}_i(\cdot)$ denote the $i$-th elements of the true and predicted observable states, respectively. In addition, the overall root-mean-square error (RMSE) of the observable-space multi-step prediction is defined as}
\begin{equation}
\label{eq:koopman_obs_rmse}
E_o=
\sqrt{
\frac{1}{M_NN_pL}
\sum_{j=1}^{M_N}\sum_{k=1}^{N_p}\sum_{i=1}^{L}
\left(\phi_i\!\left(x_k^{[j]}\right)-\hat{\phi}_i\!\left(x_k^{[j]}\right)\right)^2
},
\end{equation}
\textcolor{black}{where $L$ is the observable-space dimension.}

\textcolor{black}{Fig.~\ref{fig:koopman_openloop_verify}(a) presents a representative heatmap of observable-space multi-step prediction errors on the high-fidelity CarSim dataset. As the prediction horizon increases, the prediction error gradually accumulates, which is expected for open-loop forecasting. Nevertheless, the error remains concentrated within a relatively limited range over most observable dimensions, indicating that the learned deep Koopman model preserves satisfactory long-horizon prediction consistency.}

\textcolor{black}{Fig.~\ref{fig:koopman_openloop_verify}(b) shows the multi-step prediction errors in the original vehicle-state space on the HongQi E-HS3 real-vehicle dataset under different numbers of training trajectories. As the number of training trajectories increases, the prediction errors decrease consistently, indicating that richer training-data coverage improves the predictive accuracy and generalization capability of the learned model. Together, these results support the use of the learned deep Koopman model as the predictive model in the subsequent control framework.
}

\begin{figure}[t]
    \centering
    \subfloat[Observable-space multi-step prediction heatmap on the high-fidelity CarSim dataset.]{
        \includegraphics[width=0.95\columnwidth]{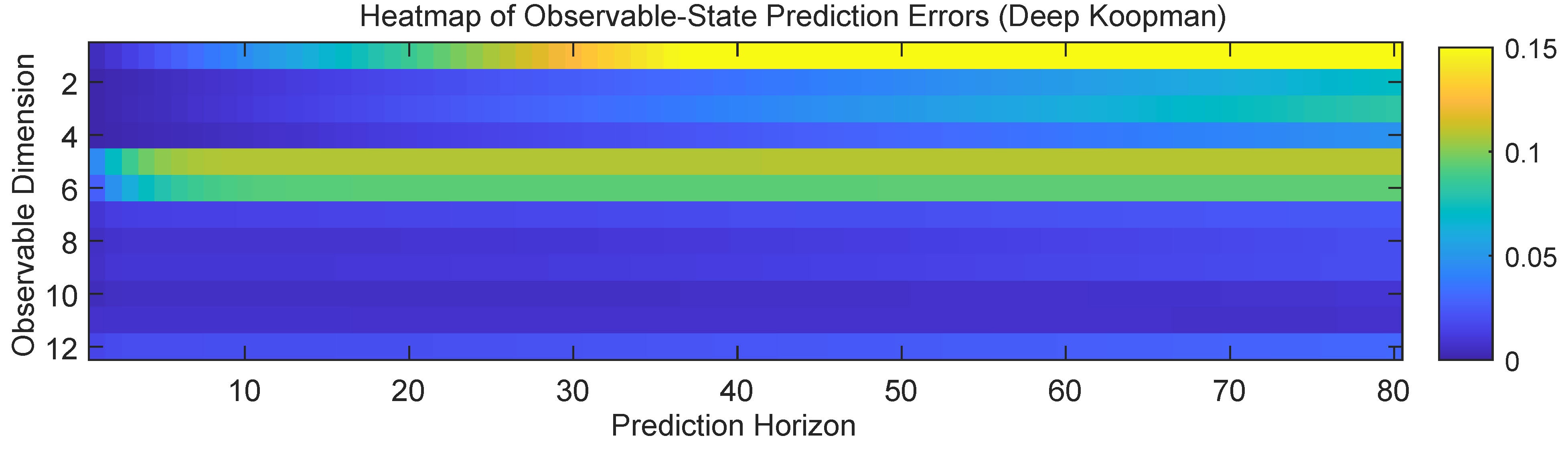}
        \label{fig:koopman_openloop_verify_a}
    }\hfill
    \subfloat[Multi-step prediction errors in the original vehicle-state space on the HongQi E-HS3 real-vehicle dataset under different numbers of training trajectories.]{
        \includegraphics[width=0.9\columnwidth]{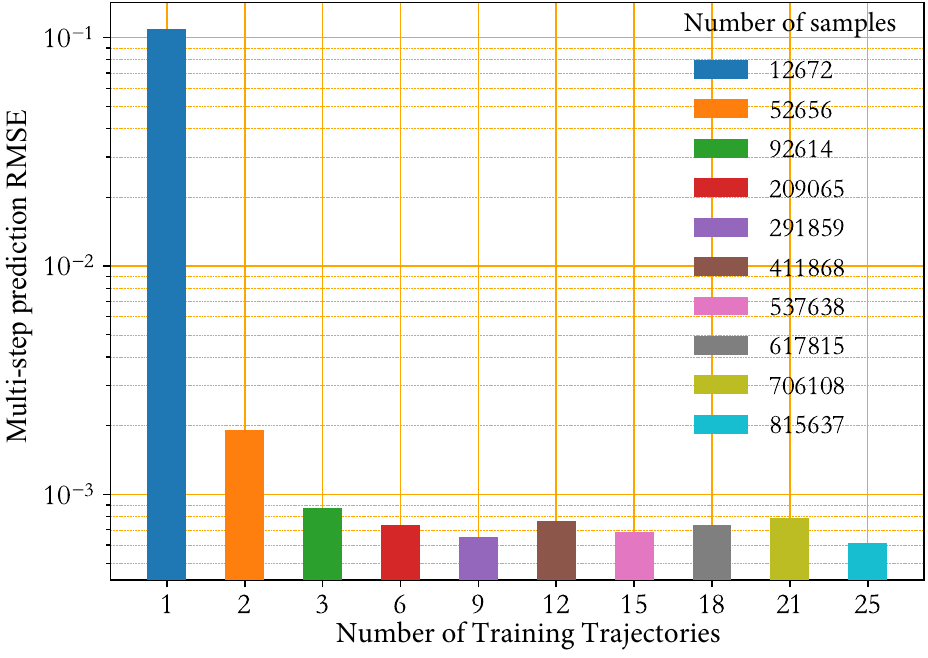}
        \label{fig:koopman_openloop_verify_b}
    }
    \caption{\textcolor{black}{Open-loop multi-step prediction verification results of the learned deep Koopman model.}}
    \label{fig:koopman_openloop_verify}
\end{figure}

\subsection{Ablation Study and Sensitivity Analysis}

\textcolor{black}{To further clarify the role of individual modules and the influence of key tuning parameters, we have also conducted an ablation study and a single-factor sensitivity analysis in a single-obstacle straight-road scenario. In this scenario, the ego vehicle traveled at 30 km/h on a 125 m two-lane straight road and avoided a static obstacle located close to the road centerline, whose center was approximately 60 m downstream from the road entrance.}

\subsubsection{Ablation Study}

\textcolor{black}{The ablation study compared the full LPC framework with three reduced variants, in which the DK predictor, the online actor--critic adaptation mechanism, or the safety term was selectively removed. The corresponding results are summarized in Table~\ref{tab:ablation_results}. The full method completes the obstacle-avoidance task without collision, whereas removing the safety term leads to collision, freezing the actor prevents successful trajectory completion, and replacing the DK predictor with a simplified kinematic model also results in failure. These results show that the performance gain arises from the synergy of the three components rather than from any single module alone: the safety term is essential for maintaining obstacle clearance, the online actor-critic adaptation is important for closed-loop correction and trajectory recovery, and the DK predictor is critical for accurate multi-step prediction and reliable policy update.}

\begin{table}[t]
\centering
\scriptsize
\setlength{\tabcolsep}{6 pt}
\renewcommand{\arraystretch}{0.98}
\begin{threeparttable}
\caption{\textcolor{black}{Ablation study results for the single-obstacle straight-road scenario.}}
\label{tab:ablation_results}
\begin{tabular}{lccccc}
\toprule
Case & Success & Collision & \makecell{StopX\\(m)} & \makecell{Min. gap\\(m)} & \makecell{Min. center\\distance (m)} \\
\midrule
DK+AC+Safety      & 1 & 0 & 120.08 & 1.55 & 3.62 \\
DK+AC             & 0 & 1 & 120.07 & 0.00 & 0.89 \\
DK+FA+Safety      & 0 & 0 &  51.35 & 6.65 & 10.48 \\
Kine.+AC+Safety   & 0 & 1 &  66.90 & 0.00 & 0.72 \\
\bottomrule
\end{tabular}
\begin{tablenotes}[flushleft]
\footnotesize
\item \emph{Remark:} DK = deep Koopman; AC = actor--critic; FA = frozen actor; Kine. = kinematic predictor.
\end{tablenotes}
\end{threeparttable}
\end{table}

\subsubsection{Sensitivity Analysis}

\textcolor{black}{We further perform a single-factor sensitivity analysis with respect to the obstacle-safety scaling factor $\beta_0$, the actor--critic update discount factor $\Gamma$, and the prediction horizon $N_p$, using the nominal setting $(\beta_0,\Gamma,N_p)=(1,1,30)$. Here, $\beta_0$ is introduced only for sensitivity analysis as a nominal scaling factor applied to the obstacle-related safety term. These three factors are selected because they represent three influential and interpretable design dimensions of the proposed controller: $\beta_0$ scales the nominal obstacle-related safety term, $\Gamma$ controls the influence of future value/costate information on the current update across the prediction interval, and $N_p$ determines the planning foresight of the receding-horizon controller. The corresponding results are reported in Table~\ref{tab:sensitivity_single_factor}. Overall, the results show that increasing $\beta_0$ enlarges obstacle clearance but also yields more conservative trajectories, increasing $\Gamma$ improves anticipatory behavior and safety margin with only a moderate change in trajectory deviation, and increasing $N_p$ does not monotonically improve safety, indicating a nontrivial trade-off between planning foresight and obstacle clearance. Among the tested settings, the nominal choice $(1,1,30)$ provides a balanced compromise between clearance and tracking performance.}

\begin{table}[!t]
\centering
\scriptsize
\setlength{\tabcolsep}{3pt}
\renewcommand{\arraystretch}{0.98}
\caption{\textcolor{black}{Single-factor sensitivity analysis results in the single-obstacle straight-road scenario.}}
\label{tab:sensitivity_single_factor}
\resizebox{\columnwidth}{!}{%
\begin{tabular}{ccccccc}
\toprule
Factor & Level & Success & Collision & \makecell{Min. gap\\(m)} & \makecell{Min. center\\distance (m)} & \makecell{Traj. tracking \\error (m)} \\
\midrule
\multirow{3}{*}{$\beta_0$}
& 0.90 & 0 & 1 & 0.000 & 1.935 & 0.246 \\
& 1.00 & 1 & 0 & 0.461 & 2.510 & 0.331 \\
& 1.10 & 1 & 0 & 1.342 & 3.239 & 0.462 \\
\midrule
\multirow{3}{*}{$\Gamma$}
& 0.90 & 0 & 1 & 0.000 & 1.588 & 0.252 \\
& 0.95 & 1 & 0 & 0.232 & 2.234 & 0.330 \\
& 1.00 & 1 & 0 & 0.461 & 2.510 & 0.331 \\
\midrule
\multirow{3}{*}{$N_p$}
& 20 & 1 & 0 & 1.547 & 3.625 & 0.452 \\
& 30 & 1 & 0 & 0.461 & 2.510 & 0.331 \\
& 40 & 0 & 1 & 0.000 & 1.969 & 0.287 \\
\bottomrule
\end{tabular}%
}
\end{table}

\subsection{Simulation Scenarios and Results}
This section evaluates the proposed method through obstacle-avoidance simulations in structured road scenarios. Other low-speed or stationary vehicles are treated as obstacle vehicles and are assumed to have the same dimensions as the HongQi-EHS3.
Four representative scenarios are considered: (\romannumeral1) static obstacle avoidance on a straight two-lane road, where the ego vehicle travels at 30~km/h between two stationary obstacles; (\romannumeral2) dynamic obstacle avoidance on a straight two-lane road, where the obstacles move at 3.6~km/h and 10.8~km/h; (\romannumeral3) overtaking on a straight road, where the ego vehicle travels at 40~km/h and overtakes a vehicle moving at 18~km/h; and (\romannumeral4) overtaking on a curved road under the same vehicle speeds as Scenario~III.

\textit{Evaluation Metrics:}
The methods are evaluated in terms of safety, real-time performance, and driving comfort. The safety metric is defined as
\begin{equation}\label{equRHDHP:safe_index}
	\left\{ \begin{array}{l}
		I_{s,f}=\frac{\bar{d}_{s,f} - l_{\text{ego}}/2 - l_{\text{obs}}/2}{d_{s,f}},\\
		I_{s,r}=\frac{\bar{d}_{s,r}}{d_{s,r}},\\
	\end{array} \right.
\end{equation}
where $l_{\text{ego}}$ and $l_{\text{obs}}$ denote the lengths of the ego and obstacle vehicles, respectively. $I_{s,f}$ measures safety when the ego vehicle starts the avoidance maneuver behind the obstacle, whereas $I_{s,r}$ measures safety when it returns to the target lane ahead of the obstacle. Here, $\bar{d}_{s,f}$ and $\bar{d}_{s,r}$ are the corresponding actual longitudinal distances, and the associated safe distances are defined as $d_{s,f}=2(v_e-v_o)$ and $d_{s,r}=2(v_e+v_o)$, where $v_e$ and $v_o$ are the speeds of the ego and obstacle vehicles, respectively. These distances are measured when the geometric center of the ego vehicle leaves and re-enters the desired lane. The overall safety metric is then
\begin{equation}\label{equRHDHP:safe_index_final}
I_s = \min \left\{ I_{s,f}^{[1]}, I_{s,r}^{[1]}, I_{s,f}^{[2]}, I_{s,r}^{[2]}, \dots, I_{s,f}^{[n_o]}, I_{s,r}^{[n_o]} \right\},
\end{equation}
where $I_{s,f}^{[i]}$ and $I_{s,r}^{[i]}$ are the safety indices associated with the $i$-th obstacle vehicle, and $n_o$ is the total number of obstacles.

The real-time performance index is defined as
\begin{equation}\label{equRHDHP:real_time_index}
	I_c=\frac{1}{(k_e-k_s)t_s}\sum_{k=k_s}^{k_e}{T_k},
\end{equation}
where $t_s$ is the sampling interval, $T_k$ is the computation time in the $k$-th control cycle, and $k_s$ and $k_e$ denote the start and end time steps of the maneuver. A value of $I_c>1$ indicates that the method cannot satisfy the real-time requirement; smaller values imply better computational efficiency.

The driving comfort index is defined as the average curvature of the obstacle-avoidance trajectory:
\begin{equation}\label{equRHDHP:realtime_index}
	I_{\rho} = \frac{1}{k_e-k_s}\sum_{k=k_s}^{k_e} \|\rho_k\|,
\end{equation}
where $\rho_k$ is the trajectory curvature at time step $k$. A smaller $I_{\rho}$ indicates better ride comfort.

\textit{Scenario I - Static Obstacle Avoidance on Straight Road:} In Scenario I, the AV navigates a straight two-lane road at 30 km/h, tracking the lane centerline while avoiding two stationary obstacle vehicles positioned near the centerline on both sides. Fig.~\ref{fig:scenario01} compares the avoidance trajectories of the proposed method, CBF-MPC, and LMPCC. All methods successfully avoid collisions, but the proposed method generates a safer and smoother path by leveraging enhanced CFS constraints, initiating avoidance earlier than benchmarks. In contrast, CBF-MPC and LMPCC delay avoidance until closer to obstacles, compromising safety (Table \ref{tabRHDHP:combined_metrics}, $I_s = 0.50$ for Ours vs. 0.08 for both benchmarks).
Furthermore, Fig.~\ref{fig:scenario01_error} shows lateral displacement and heading angle errors, revealing negligible overshoot ($<0.1 m$) for the proposed method compared to ~0.7 m for CBF-MPC and LMPCC when returning to the reference path. This smoothness contributes to a 10\% lower comfort index (Table \ref{tabRHDHP:combined_metrics}, $I_p = 0.0086$ vs. 0.0096 and 0.0093).
Additionally, the proposed method’s real-time performance is superior, with a control cycle time of ~5 ms (Table \ref{tabRHDHP:combined_metrics}, $I_c = 0.25$ vs. 2.21 and 2.29), an 89\% improvement driven by the deep Koopman model’s efficient linear dynamics and the actor-critic MPC framework’s polynomial complexity, highlighting the method’s ability to balance safety, comfort, and computational efficiency in static obstacle avoidance scenarios.

\begin{figure}[h]
    \centering
		\centering
        \subfloat[Ours]{\includegraphics[width=1\columnwidth]{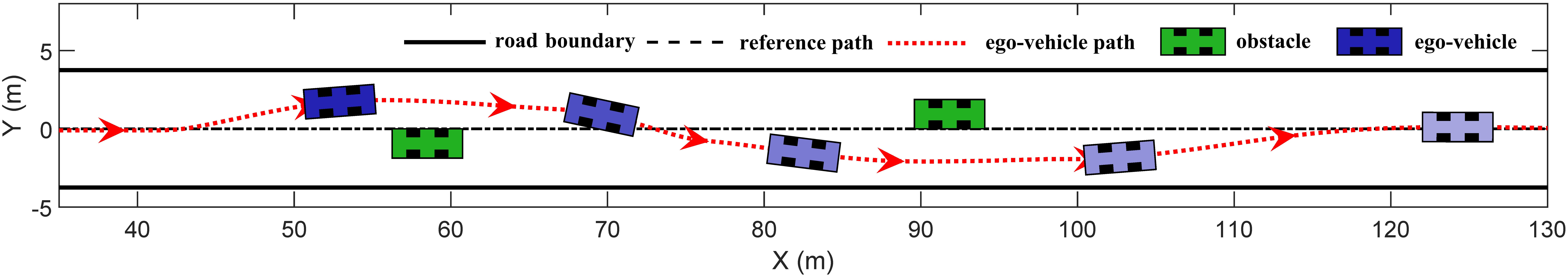}\label{fig:scenario01-1}}\quad
		\centering
        \subfloat[CBF-MPC]{\includegraphics[width=1\columnwidth]{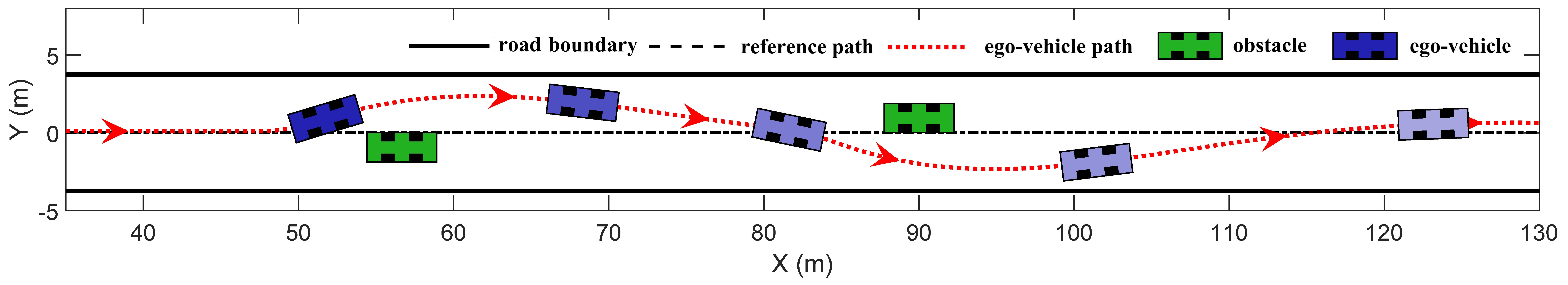}\label{fig:scenario01-2}}\quad        
		\centering		
        \subfloat[LMPCC]{\includegraphics[width=1\columnwidth]{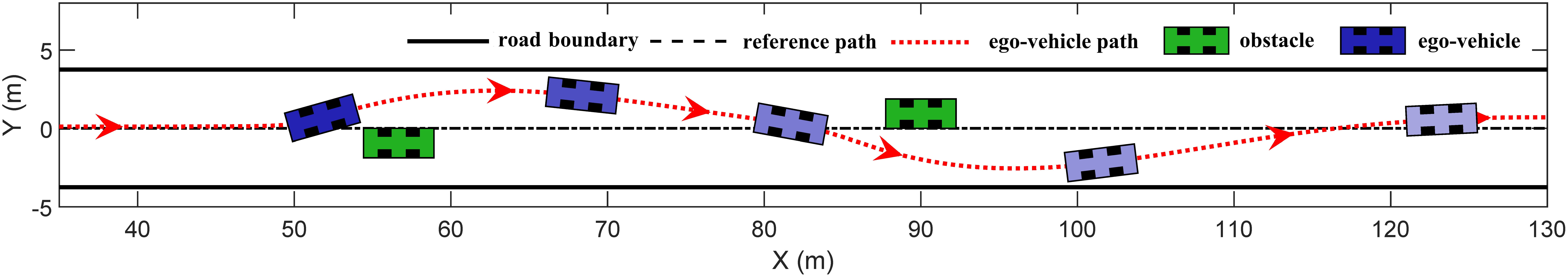}\label{fig:scenario01-3}}
	\caption{Vehicle obstacle avoidance path in Scenario I. The subfigures illustrate the avoidance path of the proposed method, CBF-MPC, and LMPCC in navigating around stationary obstacles positioned on either side of the road centerline.}
	\label{fig:scenario01}
\end{figure}

\begin{figure}
	\begin{center}
		\includegraphics[width=1\columnwidth]{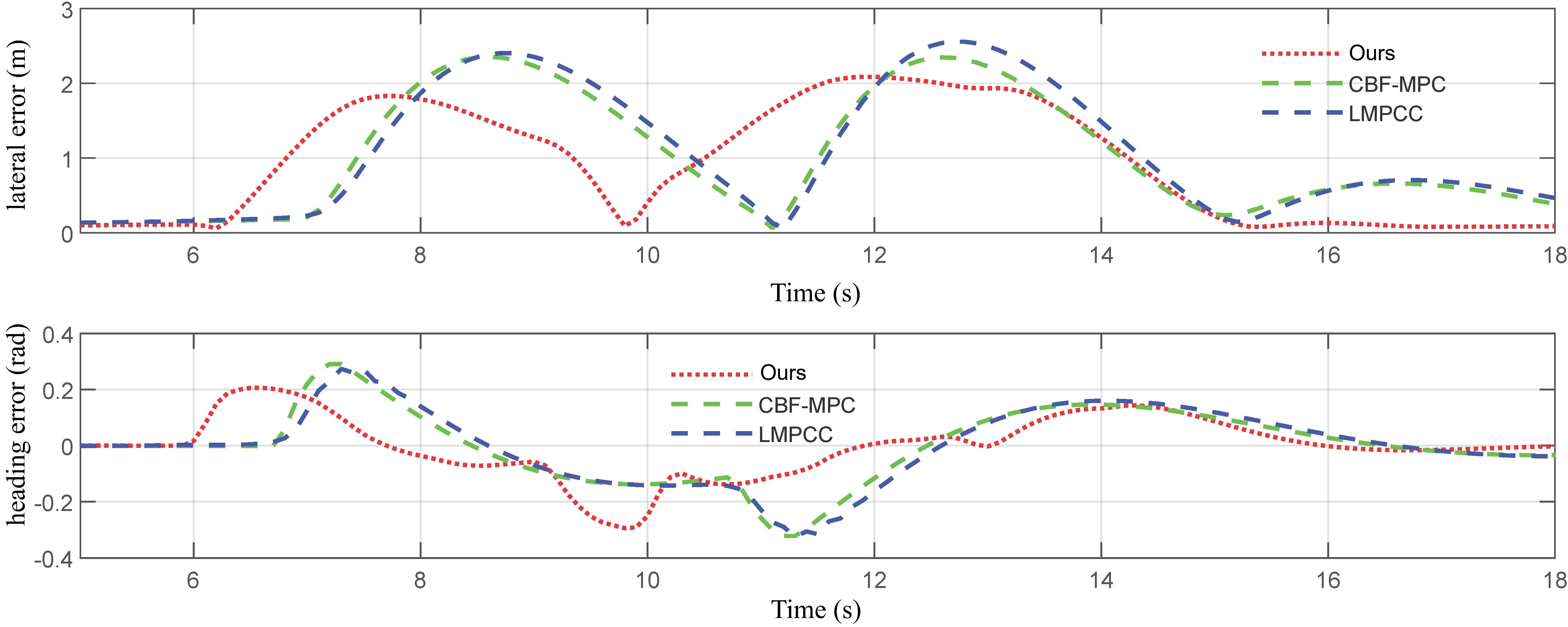}
		\caption{Lateral displacement and heading angle errors of vehicle obstacle avoidance paths in Scenario I.}
		\label{fig:scenario01_error}
	\end{center}
\end{figure}

\textit{Scenario II - Dynamic Obstacle Avoidance on Straight Road:} Scenario II challenges the AV to navigate a straight two-lane road at 30 km/h, tracking the lane centerline while dodging two dynamic obstacle vehicles moving at 3.6 km/h and 10.8 km/h on either side.  
As depicted in Fig.~\ref{fig:scenario02}, the proposed method adeptly maneuvers around these moving obstacles, maintaining a robust safety margin. Unlike CBF-MPC and LMPCC, which react late and risk closer encounters (Table \ref{tabRHDHP:combined_metrics}, $I_s = 0.40$ for Ours vs. 0.10 for both benchmarks), the proposed method anticipates obstacle motion early, thanks to its deep Koopman model’s precise dynamics prediction and potential field-driven safety constraints.
Fig.~\ref{fig:scene02_error}  illustrates the AV’s lateral displacement and heading angle errors, showcasing the proposed method’s seamless return to the reference trajectory with minimal overshoot ($<0.1m$) compared to ~0.7 m for benchmarks. This precision yields a superior comfort index (Table \ref{tabRHDHP:combined_metrics}, $I_p = 0.0074$ vs. 0.0079 and 0.0081), reflecting smoother steering adjustments. Computationally, the method remains highly efficient, solving each control cycle in ~7.4 ms (Table \ref{tabRHDHP:combined_metrics}, $I_c = 0.37$ vs. 2.32 and 2.35), an 84\% improvement over benchmarks. This efficiency stems from the actor-critic framework’s ability to rapidly adapt policies in the Koopman observation space, ensuring real-time responsiveness even with dynamic obstacles. These results underscore the strength of the method to handle dynamic obstacles while prioritizing safety and ride quality.

\begin{figure}[h]
    \centering
		\centering
        \subfloat[Ours]{\includegraphics[width=1\columnwidth]{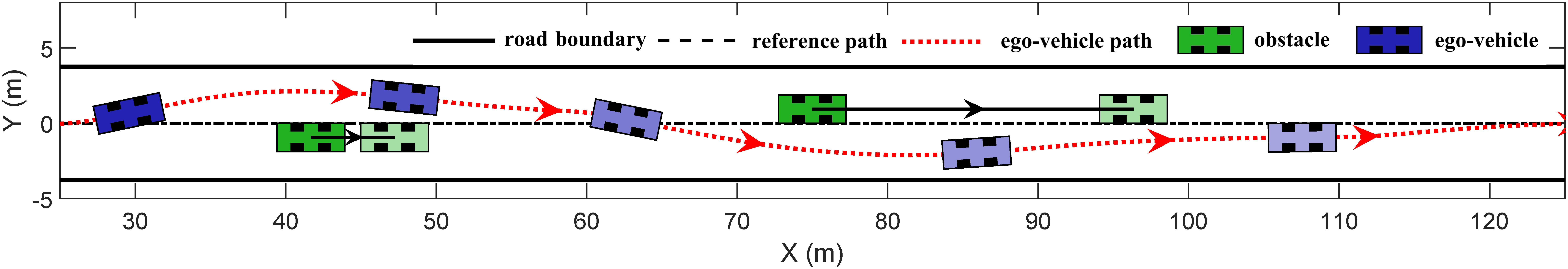}\label{fig:scenario02-1}}\quad
		\centering
        \subfloat[CBF-MPC]{\includegraphics[width=1\columnwidth]{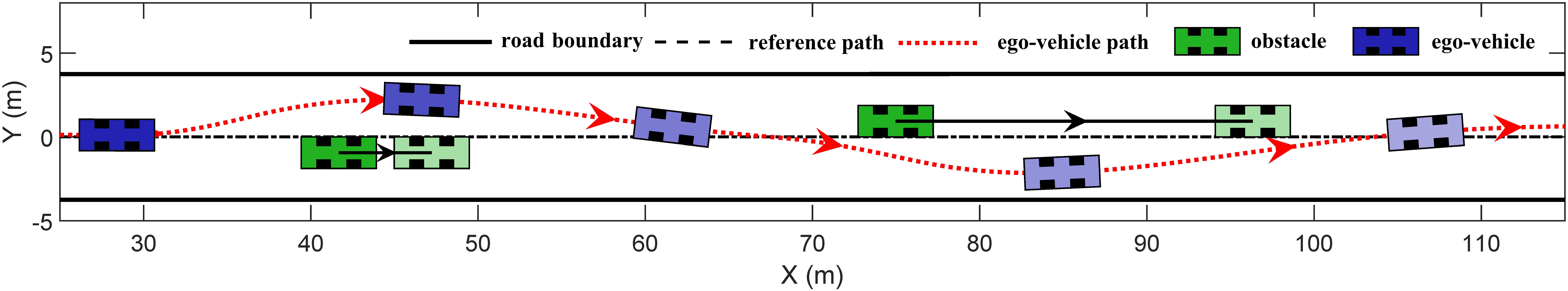}\label{fig:scenario02-2}}\quad        
		\centering		
        \subfloat[LMPCC]{\includegraphics[width=1\columnwidth]{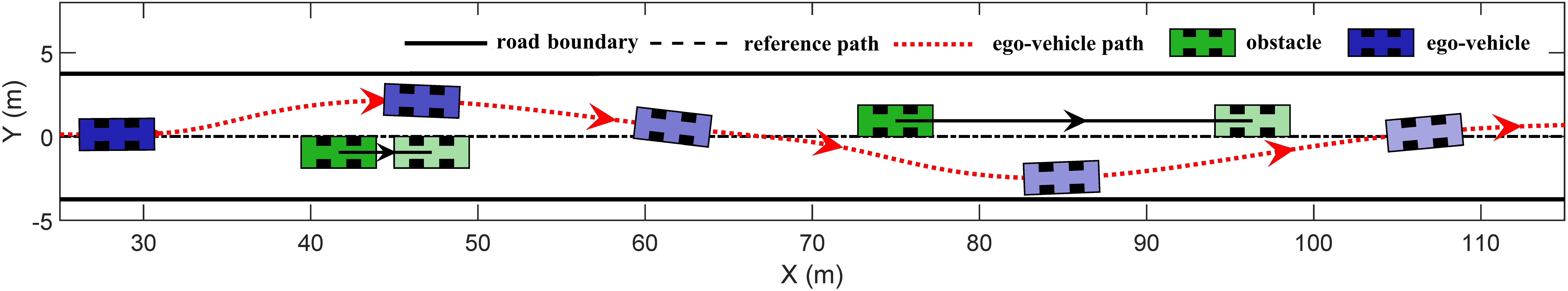}\label{fig:scenario02-3}}
	\caption{Vehicle obstacle avoidance path in Scenario II. The subfigures illustrate the avoidance path of the proposed method, CBF-MPC, and LMPCC in navigating around moving obstacle vehicles on a straight road.}
	\label{fig:scenario02}
\end{figure}

\begin{figure}[htb]
    \begin{center}
        \includegraphics[width=1\columnwidth]{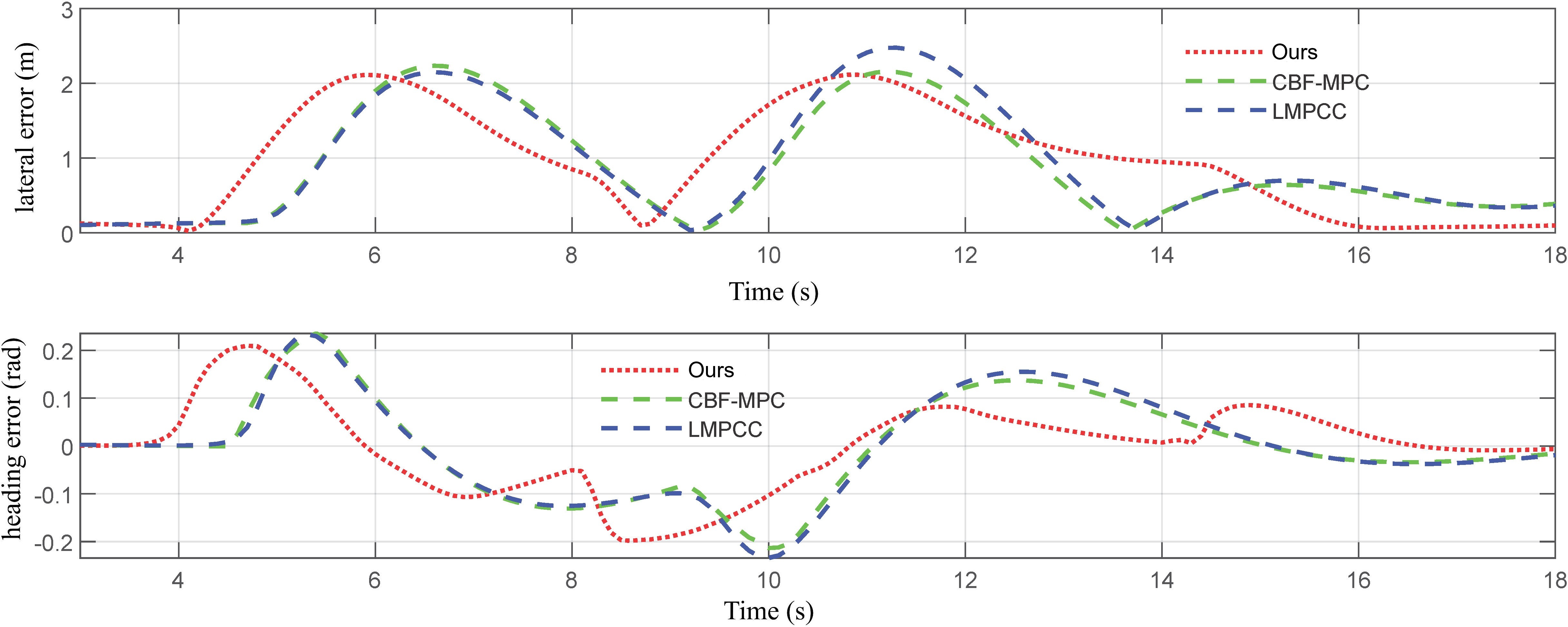}
        \caption{Lateral displacement and heading angle errors of vehicle obstacle avoidance paths in Scenario II.}
        \label{fig:scene02_error}
    \end{center}
\end{figure}

\textit{Scenario III-Overtaking Maneuvering on Straight Road:}
Scenario III tests the AV at 40 km/h overtaking a slower obstacle vehicle traveling at 18 km/h on a straight two-lane road, requiring precise lane changes to ensure safety. 
Fig.~\ref{fig:scenario1} illustrates the overtaking paths, with the proposed method maintaining the overtaking lane longer to secure a safe longitudinal distance before returning to the desired lane. In contrast, CBF-MPC and LMPCC revert prematurely, risking proximity to the obstacle (Table \ref{tabRHDHP:combined_metrics}, $I_s = 0.38$ for Ours vs. 0.01 and 0.03).
As shown in Fig.~\ref{fig:scenario1_error}, the proposed method achieves a smooth lane return with minimal lateral displacement and heading angle deviations, avoiding the ~0.7 m overshoot observed in benchmarks. This precision delivers a 27\% better comfort index (Table \ref{tabRHDHP:combined_metrics}, $I_p = 0.0072$ vs. 0.0099 and 0.010). Computationally, the method sustains high efficiency, with a control cycle time of ~7.4 ms (Table \ref{tabRHDHP:combined_metrics}, $I_c = 0.37$ vs. 2.34 and 2.38), an 84\% improvement driven by the deep Koopman model’s linear dynamics and the actor-critic framework’s rapid policy updates. These results demonstrate the method’s adeptness during overtaking maneuvers, balancing safety, smoothness, and real-time performance.

\begin{figure}[h]
    \centering
		\centering
        \subfloat[Ours]{\includegraphics[width=1\columnwidth]{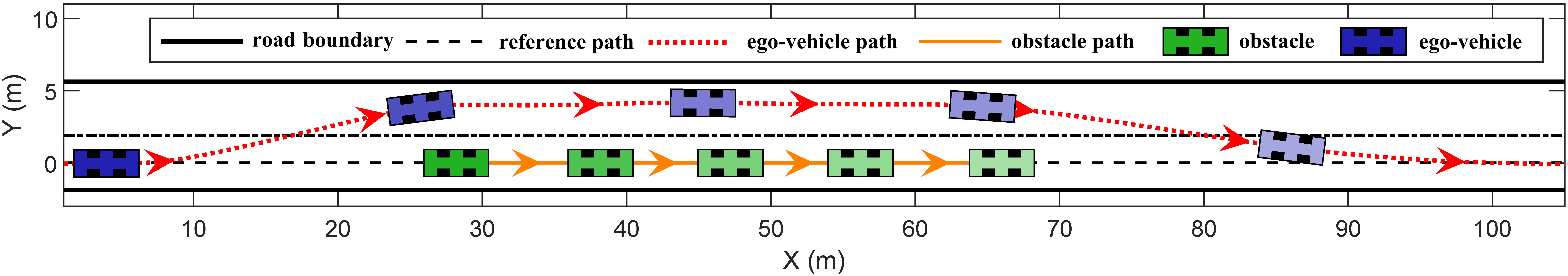}\label{fig:scenario1-1}}\quad
		\centering
        \subfloat[CBF-MPC]{\includegraphics[width=1\columnwidth]{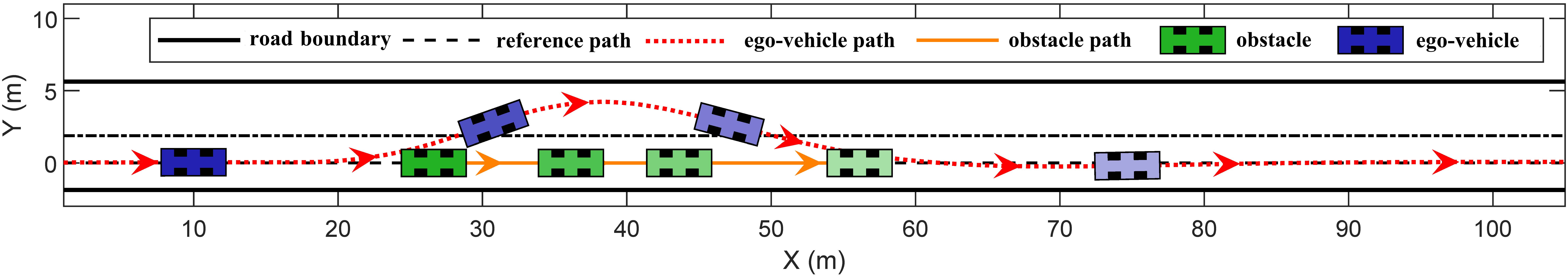}\label{fig:scenario1-2}}\quad        
		\centering		
        \subfloat[LMPCC]{\includegraphics[width=1\columnwidth]{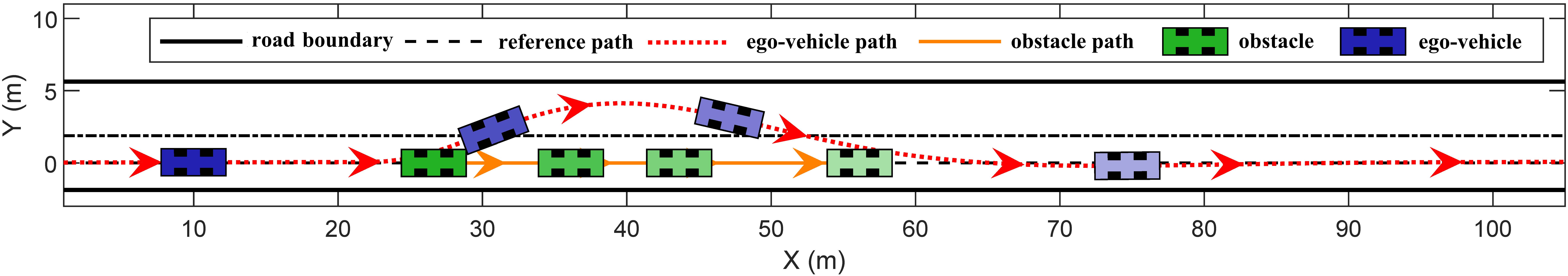}\label{fig:scenario1-3}}
	\caption{Vehicle overtaking paths in Scenario III, the subfigures illustrate the overtaking paths generated by the proposed method, CBF-MPC, and LMPCC as they deal with dynamic obstacles in the straight road.}
	\label{fig:scenario1}
\end{figure}

\begin{figure}
	\begin{center}
		\includegraphics[width=0.9\columnwidth]{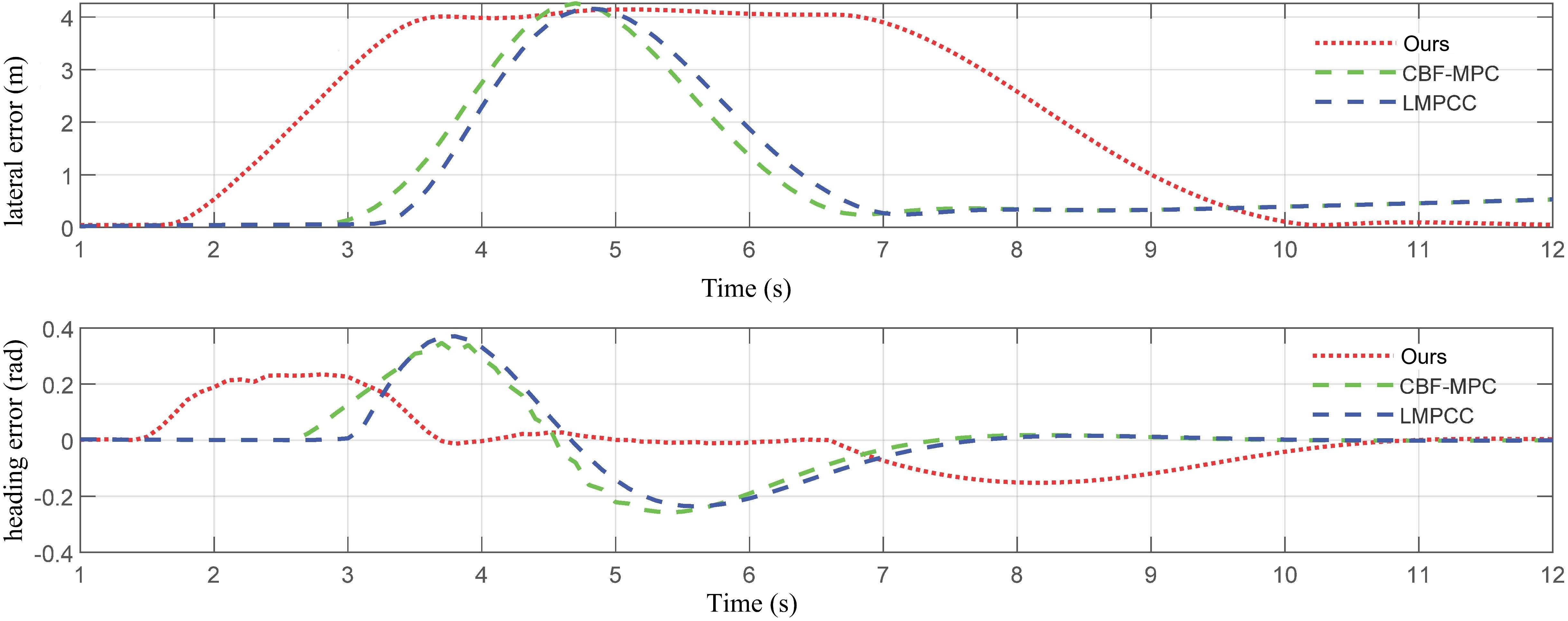}
		\caption{The lateral displacement and heading angle deviation of the vehicle's obstacle avoidance paths in Scenario III.}
		\label{fig:scenario1_error}
	\end{center}
\end{figure}

\textit{Scenario IV - Overtaking maneuver on the curved road:}
Scenario IV evaluates the AV overtaking an obstacle vehicle moving at 18 km/h on a curved two-lane road while maintaining 40 km/h, demanding precise trajectory planning to navigate road curvature and ensure safety. 
Fig.~\ref{fig:scenario2} shows the proposed method’s overtaking path, which prudently extends the overtaking lane duration to maintain a safe longitudinal gap before lane return. Conversely, CBF-MPC and LMPCC hastily revert to the desired lane, compromising safety (Table~\ref{tabRHDHP:combined_metrics}, $I_s = 0.25$ for Ours vs. 0.01 and 0.07). Fig.~\ref{fig:scenario2_error} highlights the proposed method’s stable lateral displacement and reduced heading angle fluctuations, contrasting with benchmarks’ larger deviations. This stability yields a 7\% improved comfort index (Table~\ref{tabRHDHP:combined_metrics}, $I_p = 0.013$ vs. 0.014 for both benchmarks), reflecting smoother navigation of the curve. The method also maintains computational agility, with a control cycle time of ~8 ms (Table~\ref{tabRHDHP:combined_metrics}, $I_c = 0.40$ vs. 2.36 and 2.33), an 83\% improvement driven by the actor-critic framework’s efficient policy optimization. These results affirm the method’s capability to master complex overtaking on curved roads, ensuring safety, comfort, and real-time performance.

\begin{figure}[h]
    \centering
		\centering
        \subfloat[Ours]{\includegraphics[width=0.28\columnwidth]{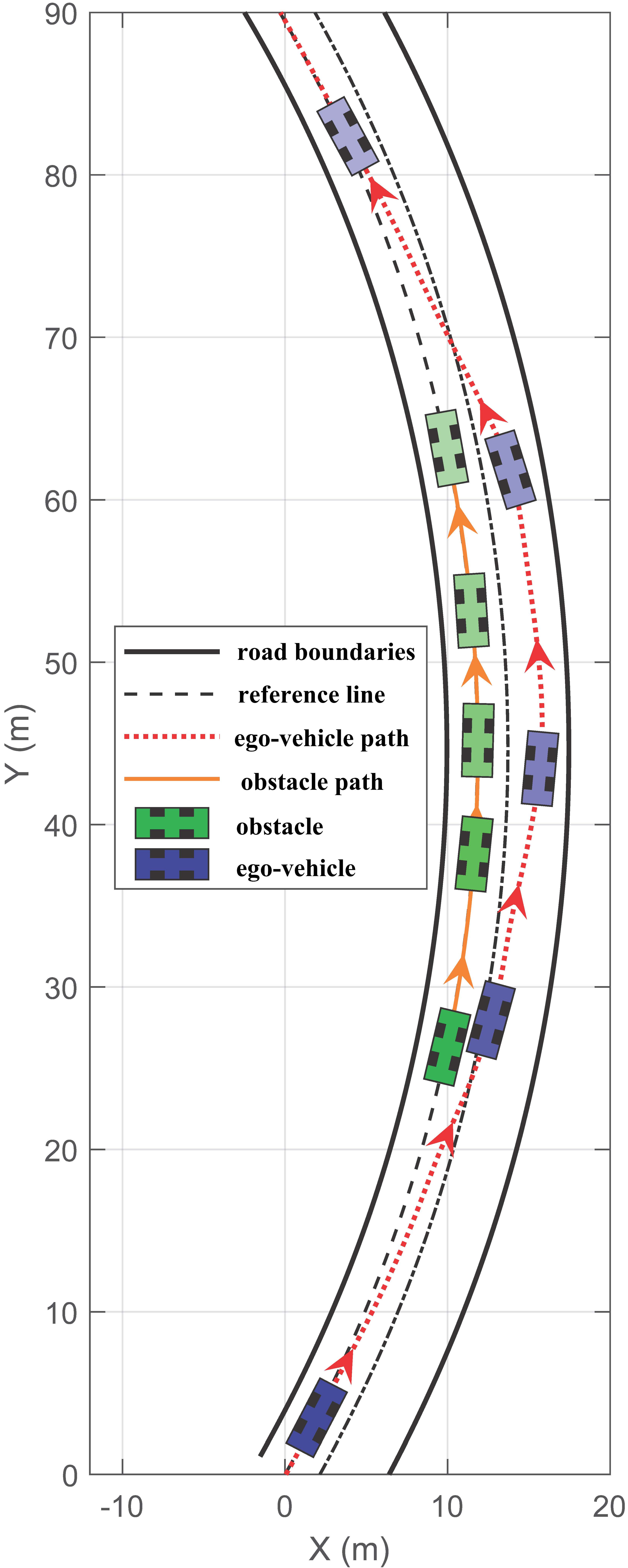}\label{fig:scenario2-1}}\quad
		\centering
        \subfloat[CBF-MPC]{\includegraphics[width=0.28\columnwidth]{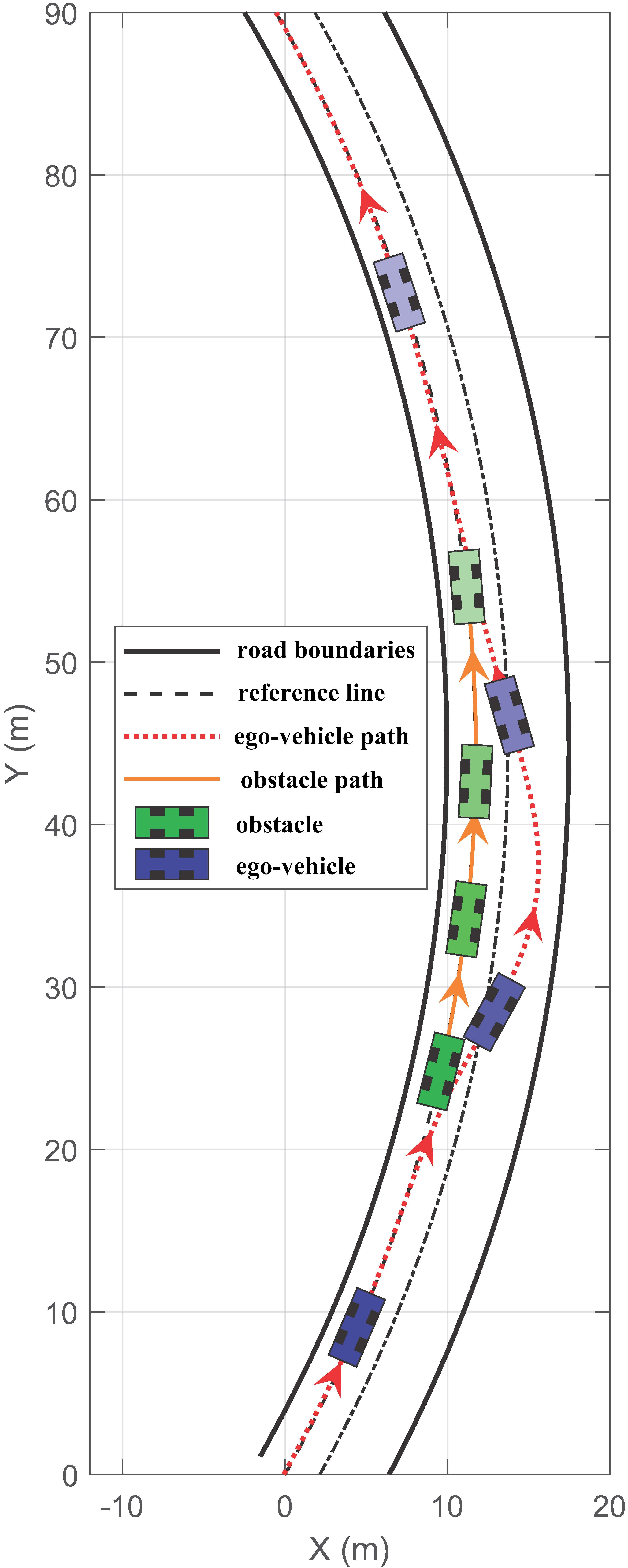}\label{fig:scenario2-2}}\quad        
		\centering		
        \subfloat[LMPCC]{\includegraphics[width=0.28\columnwidth]{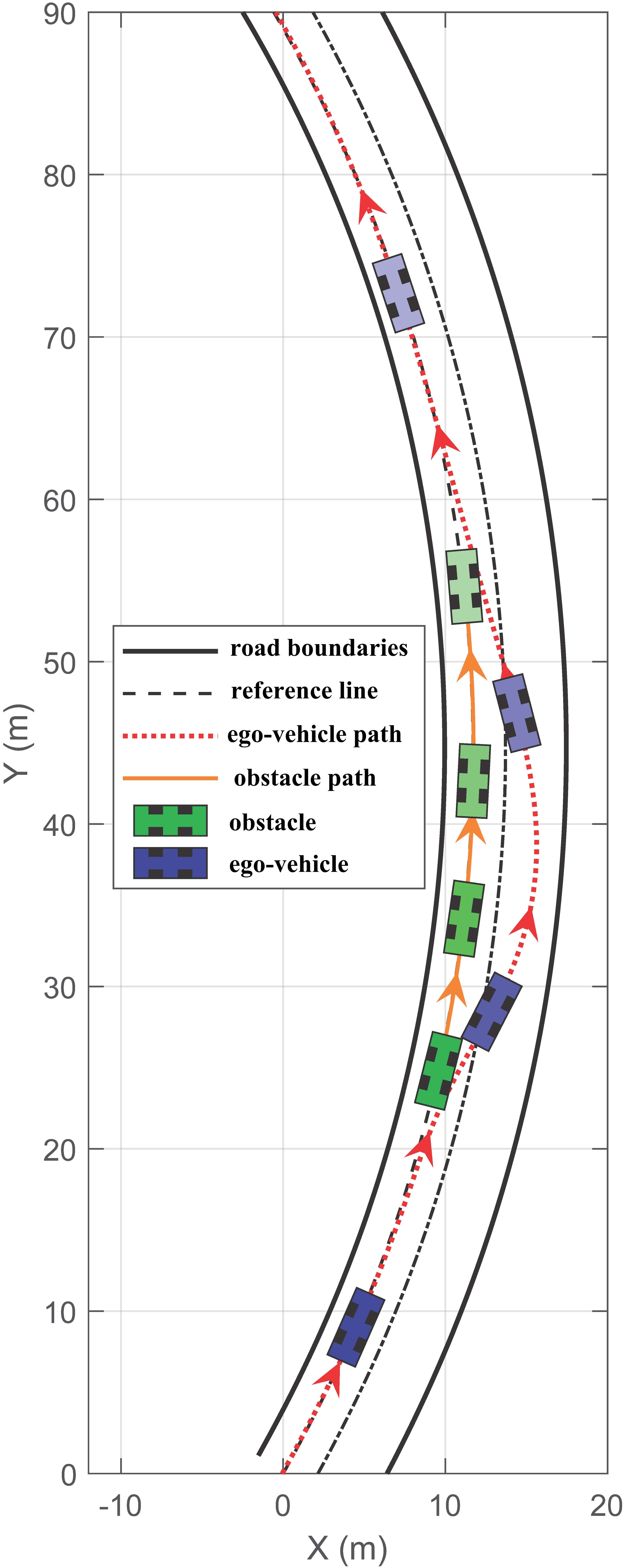}\label{fig:scenario2-3}}
	\caption{Vehicle overtaking paths in Scenario IV, the subfigures illustrate the overtaking paths generated by the proposed method, CBF-MPC, and LMPCC as they deal with dynamic obstacles in a curved road.}
	\label{fig:scenario2}
\end{figure}

\begin{figure}[htb]
	\begin{center}
		\includegraphics[width=1\columnwidth]{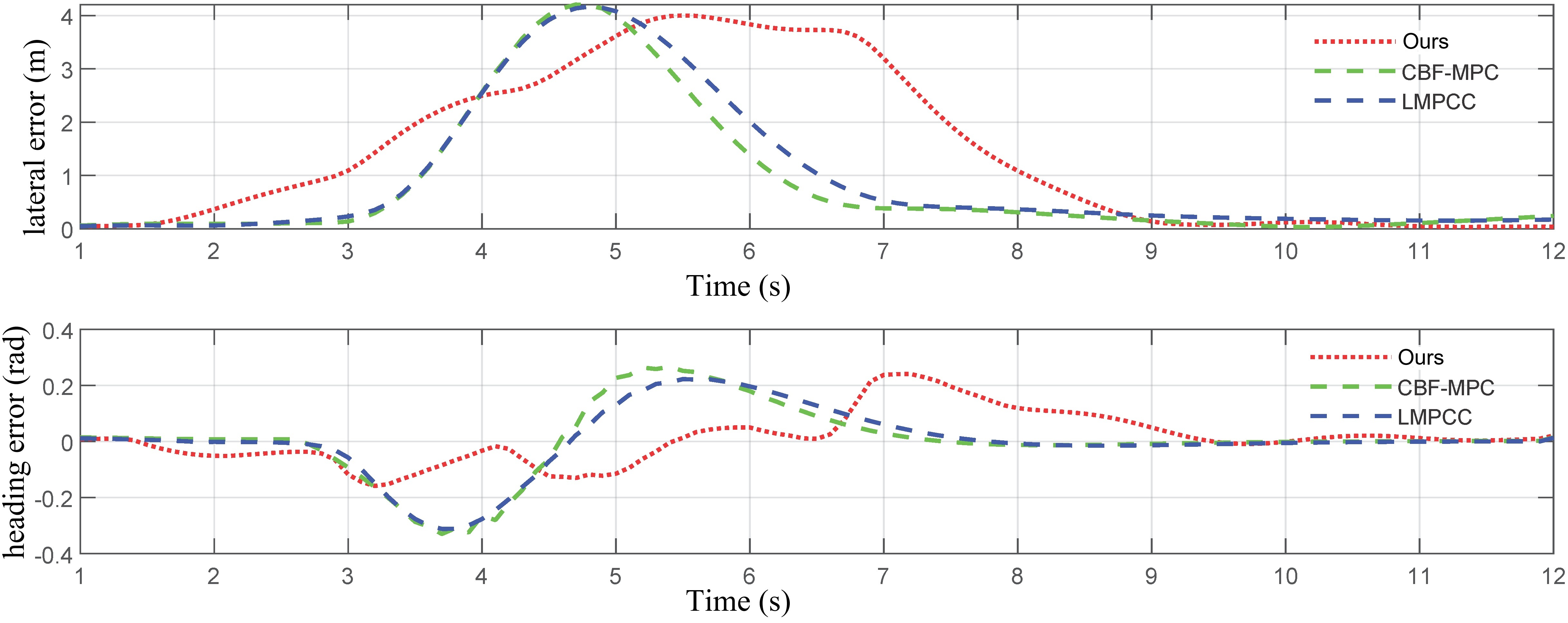}
		\caption{The lateral displacement and heading angle deviation of the vehicle's obstacle avoidance trajectory in Scenario IV.}
		\label{fig:scenario2_error}
	\end{center}
\end{figure}


\begin{table}[htbp]
    \centering
    \caption{Comparison of Evaluation Metrics for the proposed, CBF-MPC, and LMPCC in Scenarios I-IV}
    \label{tabRHDHP:combined_metrics}
    \begin{tabular}{ccccc}
        \toprule
        Scenario & Metric & Ours & CBF-MPC & LMPCC \\
        \midrule
        \multirow{3}{*}{I} 
        & Safety $I_s$ & \textbf{0.50} & 0.08 & 0.08 \\
        & Real-time $I_c$ & \textbf{0.25} & 2.21 & 2.29 \\
        & Driving comfort $I_\rho$ & \textbf{0.0086} & 0.0096 & 0.0093 \\
        \midrule
        \multirow{3}{*}{II} 
        & Safety $I_s$ & \textbf{0.40} & 0.10 & 0.10 \\
        & Real-time $I_c$ & \textbf{0.37} & 2.32 & 2.35 \\
        & Driving comfort $I_\rho$ & \textbf{0.0074} & 0.0079 & 0.0081 \\
        \midrule
        \multirow{3}{*}{III} 
        & Safety $I_s$ & \textbf{0.38} & 0.01 & 0.03 \\
        & Real-time $I_c$ & \textbf{0.37} & 2.34 & 2.38 \\
        & Driving comfort $I_\rho$ & \textbf{0.0072} & 0.0099 & 0.010\\
        \midrule
        \multirow{3}{*}{IV} 
        & Safety $I_s$ & \textbf{0.25} & 0.01 & 0.07 \\
        & Real-time $I_c$ & \textbf{0.40} & 2.36 & 2.33 \\
        & Driving comfort $I_\rho$ & \textbf{0.013} & 0.014 & 0.014 \\
        \bottomrule
    \end{tabular}
\end{table}

\begin{figure}[htb]
	\begin{center}
		\includegraphics[width=1\columnwidth]{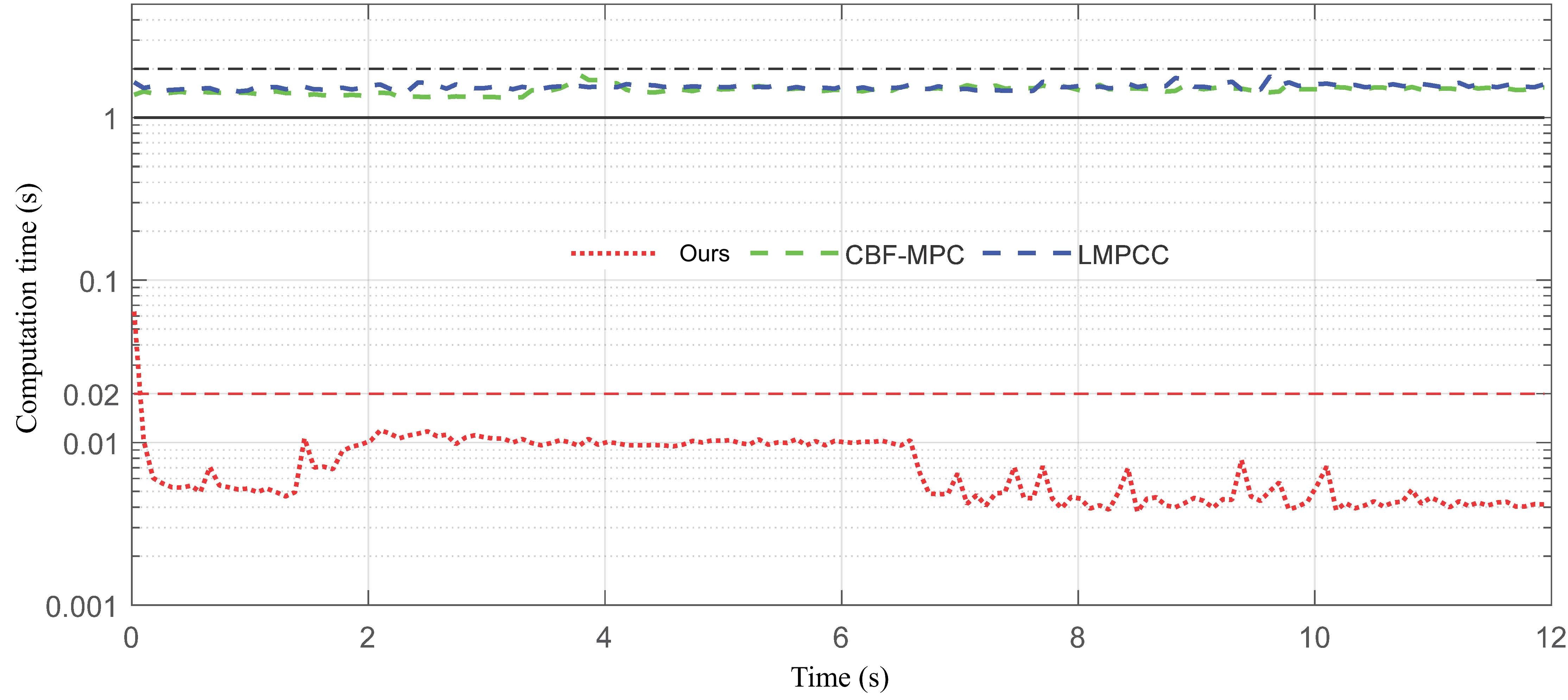}
		\caption{Computational time of different methods per time step.}
		\label{fig:time_costs}
	\end{center}
\end{figure}


To assess computational efficiency across Scenarios I–IV, Fig.~\ref{fig:time_costs} compares the per-cycle computation times of the proposed method, CBF-MPC, and LMPCC, with a 20 ms threshold (red dashed line) indicating real-time feasibility. The proposed method consistently achieves cycle times of 5–8 ms ($I_c = 0.25–0.40$), well below the threshold, while CBF-MPC and LMPCC average ~1.5 s ($I_c = 11.70–11.89$), exceeding real-time limits by over an order of magnitude. Results indicate an 84.7\% average improvement in $I_c$ reflecting the deep Koopman model’s streamlined linear dynamics and the actor-critic framework’s polynomial complexity. These efficiencies enable satisfying real-time performance, which is critical for dynamic AV motion planning, unlike the computationally intensive nonlinear optimization of benchmarks.

\subsection{Real-World Experiment Validation}

\begin{figure}  
    \centering
    \includegraphics[width=0.99\columnwidth]{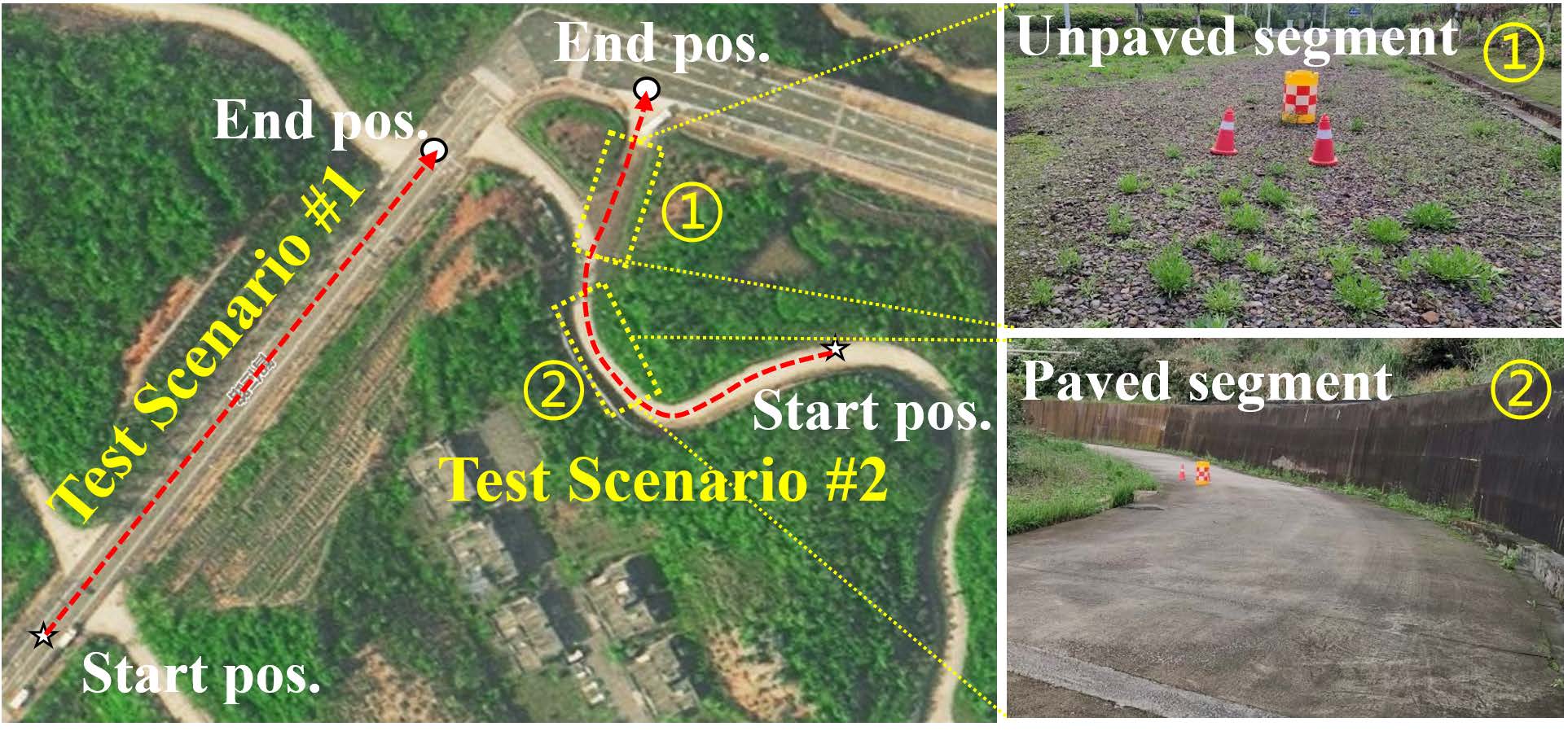}
    \caption{Real-world experiment scenarios description.}
    \label{fig:real_test_scenarios}
\end{figure}

\textcolor{black}{Two representative real-world experiment scenarios were designed to further evaluate the practical applicability of the proposed method on a Hongqi EHS3 platform equipped with an Intel i9-12900H processor and 32\,GB RAM. During the real-world tests, the system operated with a sampling interval of $t_s = 0.02~\mathrm{s}$. As illustrated in Fig.~\ref{fig:real_test_scenarios}, scenario~1 was conducted on a structured road and involves a successive double-obstacle avoidance task on a straight-line segment. Scenario~2 was designed as a more challenging unseen test scenario that is not fully covered by the training data distribution. It consisted of a dual obstacle-avoidance task along a road segment with continuously varying curvature and a transition from a paved surface to an unpaved gravel surface. In particular, the surface transition and the unevenness of the unpaved road introduce additional disturbances to vehicle motion, which provides a more stringent test of the robustness and stability of the proposed method under real-world condition variations. In the real-vehicle implementation, the controller hyperparameters were set to $\gamma_o = 19$, $g = 0.03$, $\alpha_c = 0.51$, and $\alpha_a = 0.003$, while the actor and critic network structures were kept the same as those used in simulation. Cone-shaped barrels were employed as static proxies for obstacle vehicles, and CBF-MPC was implemented as the baseline method for comparison.}

\textcolor{black}{For safety considerations, a more conservative obstacle representation was uniformly adopted in the experiment, using inflated obstacle envelopes for all methods. This conservative setting may increase the measured longitudinal clearance at lane departure and lane re-entry, and thus lead to larger absolute safety-index values than those observed in simulation.}

\begin{table}[htbp]
    \centering
    \caption{Comparison of real-vehicle evaluation metrics between the proposed method and the baseline (CBF-MPC)}
    \label{tabRHDHP:real_world_metrics}
    \scalebox{0.9}{
    \begin{tabular}{p{1 cm}cccc}
        \toprule
        Scenario & Method & Safety $I_s$ & Real-time $I_c$ & Driving Comfort $I_\rho$ \\
        \midrule
        \multirow{2}{=}{Straight road}
        & Ours    & 0.950 & 0.087 & 0.0339 \\
        & CBF-MPC & 1.259 & 2.230 & 0.0379 \\
        \midrule
        \multirow{2}{=}{Curved road}
        & Ours    & 0.301 & 0.104 & 0.1218 \\
        & CBF-MPC & 0.290 & 2.380 & 0.1394 \\
        \bottomrule
    \end{tabular}}
\end{table}

\begin{figure*}  
    \centering
    \includegraphics[width=\textwidth]{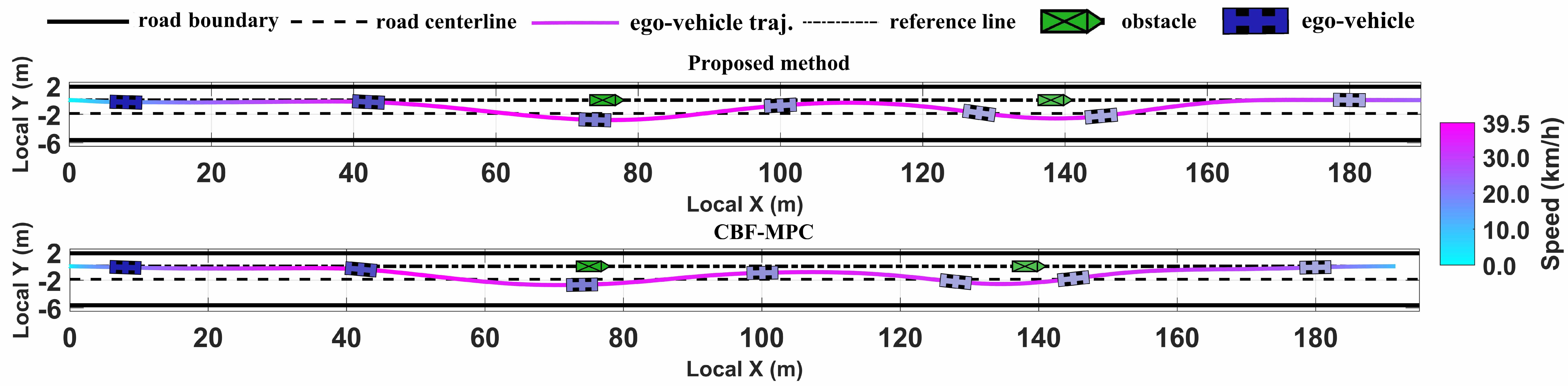}
    \caption{Real-world experiments: vehicle obstacle avoidance trajectories on a straight road (Proposed vs. CBF-MPC).}
    \label{fig:real_test_1}
\end{figure*}

\subsubsection{Test Scenario 1-- obstacle avoidance on a straight road}
\textcolor{black}{The first experiment considered a successive double-obstacle avoidance task on a straight structured road, in which the ego vehicle starts from the left lane and reaches a peak speed close to $40~\mathrm{km/h}$ during the maneuver. As shown in Fig.~\ref{fig:real_test_1}, the ego vehicle successively avoids two static obstacles and then returns to the target lane. The in-vehicle and outside-view snapshots in Fig.~\ref{fig:real_world_expe_all}~(a) further illustrate the key stages of the maneuver, including obstacle perception, lane departure, intermediate recovery, and final return.}
\textcolor{black}{The proposed method completed the maneuver safely and smoothly. As reported in Table~\ref{tabRHDHP:real_world_metrics}, it achieved a safety index of $I_s=0.950$, a real-time index of $I_c=0.087$, and a comfort index of $I_{\rho}=0.0339$. In addition, the maximum lateral acceleration during the maneuver was about $2.10~\mathrm{m/s^2}$, indicating that the real-vehicle test involved a non-negligible lateral dynamic response rather than a purely kinematic case.}


\begin{figure*}[t]
    \centering
    \subfloat[Obstacle avoidance on a straight structured road.]{
        \includegraphics[width=0.46\textwidth]{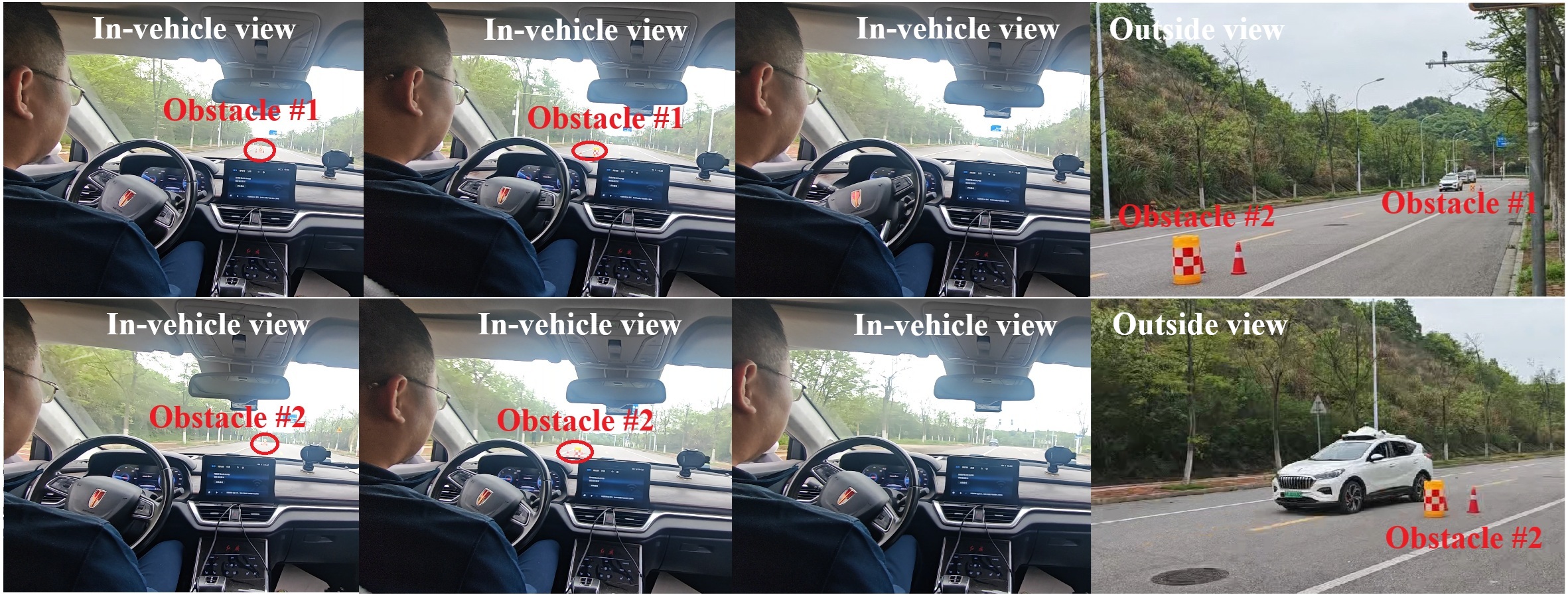}
        \label{fig:real_world_expe_sc1}
    }\hfill
    \subfloat[Obstacle avoidance on a curved-road scenario with surface transition.]{
        \includegraphics[width=0.51\textwidth]{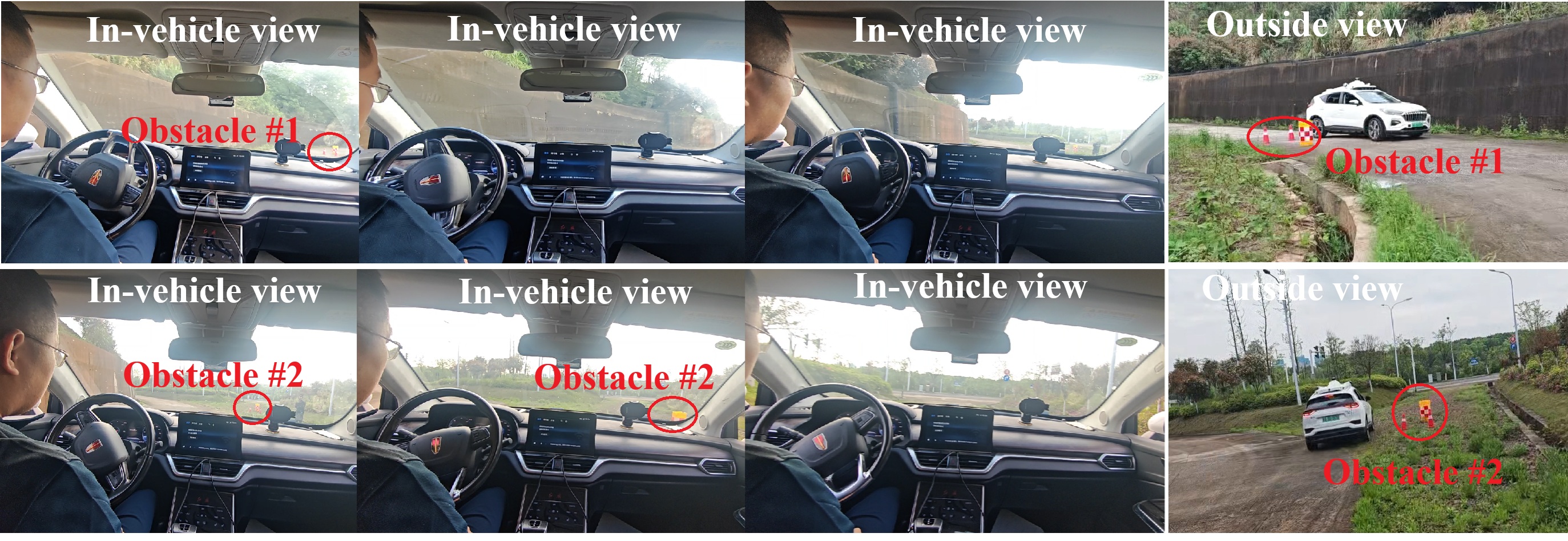}
        \label{fig:real_world_expe_sc2}
    }
    \caption{In-vehicle and outside-view snapshots of the proposed method during the two real-world test scenarios.}
    \label{fig:real_world_expe_all}
\end{figure*}

\textcolor{black}{Under the same experimental protocol, CBF-MPC achieved a larger safety index of $I_s=1.259$, but its real-time performance deteriorated substantially with $I_c=2.23$, as well as a worse comfort index with $I_{\rho}=0.0379$. These results indicate that, although the baseline method maintained a larger longitudinal safety margin in this test, the proposed method provided a more favorable overall balance among safety, real-time feasibility, and ride comfort, while preserving stable and smooth obstacle-avoidance behavior in the real-vehicle experiment.}

\subsubsection{Test Scenario 2--obstacle avoidance on a curved-road scenario with surface transition}

\textcolor{black}{The second experiment considered a more challenging obstacle-avoidance task in an unseen real-world scenario that is only partially covered by the training-data distribution. Unlike Test Scenario~1 on a straight structured road, this test involves two successive obstacle-avoidance maneuvers along a curved road with continuously varying curvature and a paved-to-unpaved gravel transition. As shown in Fig.~\ref{fig:real_test_2}, the first maneuver is completed on the paved segment, whereas the second is executed after entering the gravel segment. The in-vehicle and outside-view snapshots in Fig.~\ref{fig:real_world_expe_all}~(b) further illustrate the key stages of the maneuver and the corresponding obstacle locations. These road-geometry and surface changes introduce additional disturbances and sim-to-real mismatch, making this scenario a more demanding test of collision-avoidance robustness and closed-loop stability.}

\begin{figure*} 
    \centering 
    \includegraphics[width=1.8\columnwidth]{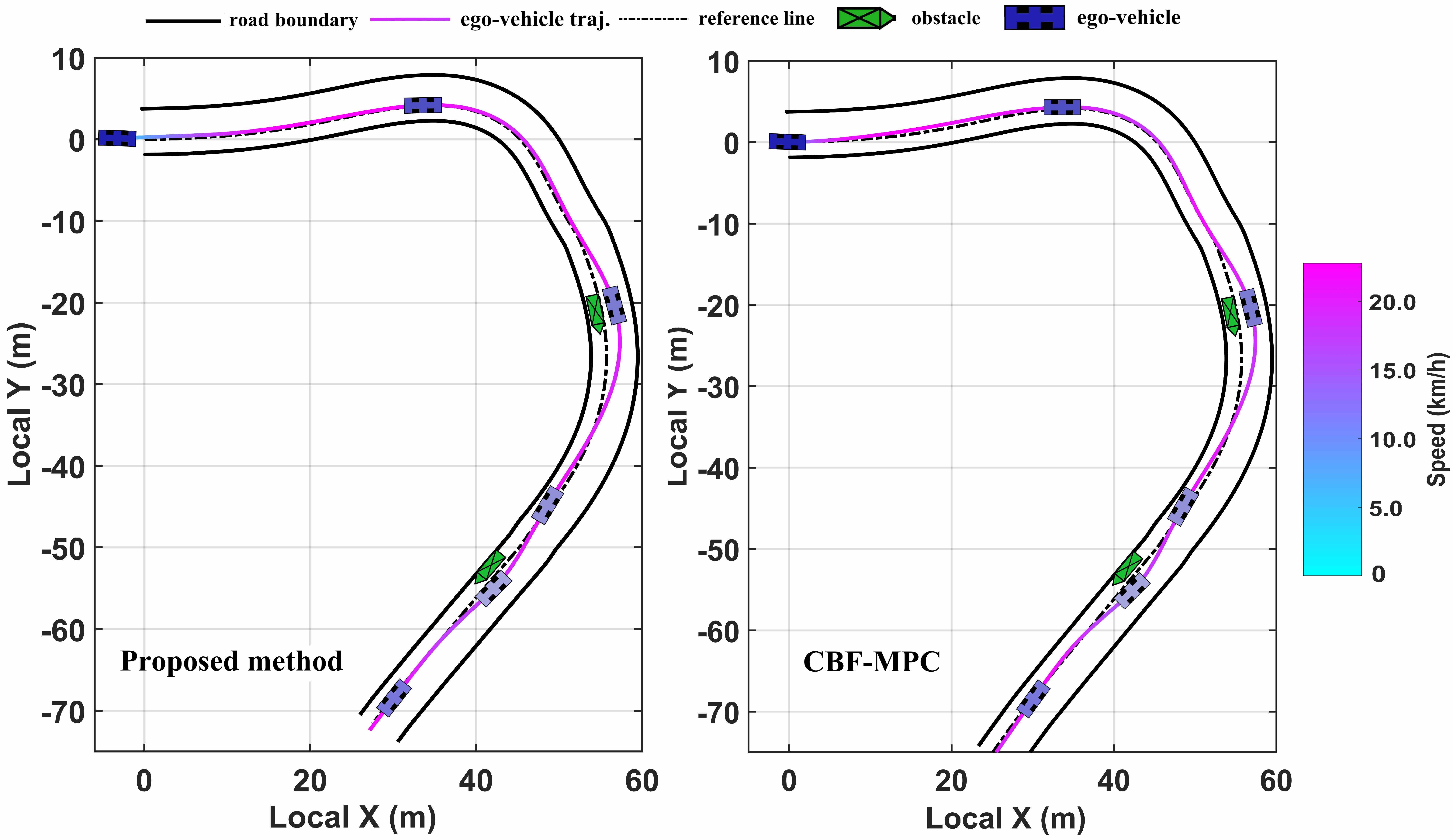} 
    \caption{Real-world experiments: vehicle obstacle avoidance trajectories on a curved road with surface transition (Proposed vs. CBF-MPC).} 
    \label{fig:real_test_2} 
\end{figure*}

\textcolor{black}{The proposed method completed the maneuver safely and without loss of stability. According to Table~\ref{tabRHDHP:real_world_metrics}, it achieved a safety index of $I_s=0.301$, a real-time index of $I_c=0.104$, and a comfort index of $I_{\rho}=0.1218$. The maximum lateral acceleration during the maneuver was approximately $2.5~\mathrm{m/s^2}$, indicating that the real-vehicle test already involved a noticeable lateral dynamic response rather than a purely kinematic operating condition.}
\textcolor{black}{Fig.~\ref{fig:real_test_2} shows that the proposed method maintains a smooth and spatially consistent avoidance trajectory throughout the maneuver, with no evident oscillatory correction or abrupt path distortion near the obstacle locations, even after entering the gravel segment. Under the same experimental protocol, CBF-MPC also completed the task and achieved a safety index close to that of the proposed method, but its real-time performance was substantially worse and its trajectory smoothness was slightly inferior. Together with the results in Table~\ref{tabRHDHP:real_world_metrics}, these observations indicate that the proposed method provides a more favorable overall balance among safety, real-time feasibility, and trajectory smoothness in this unseen scenario.}

\subsection{Discussion and Limitations}

\subsubsection{Sim-to-Real Transfer and Operating-Condition Changes}
\textcolor{black}{
The deep Koopman predictor provides an efficient lifted-space surrogate for nonlinear vehicle dynamics, while remaining a data-driven approximation of the real system. In the proposed framework, the effect of model mismatch is alleviated by the receding-horizon closed-loop implementation: only the first control input is applied at each sampling instant, and the optimization is repeated using updated onboard sensing information. In addition, the actor--critic structure yields an explicit state-feedback policy within each prediction interval, which further improves online correction capability against prediction mismatch.
}

\textcolor{black}{
The real-world experiments on the HongQi-EHS3 platform provide practical evidence of transferability from simulation-trained modeling/control to real-vehicle scenarios. The tested cases include structured-road obstacle avoidance as well as obstacle avoidance on a curved road containing a transition between structured pavement and unstructured gravel surface, with vehicle speed reaching up to 40 km/h. Together, these experiments extend the empirical validation of the proposed method to more challenging road-surface conditions, road curvature, and higher-speed operation.
}

\subsubsection{Current Limitations}
\textcolor{black}{
Several limitations of the present study should be noted. First, the effectiveness of the proposed LPC framework remains related to the operating domain covered by the training and validation data, since the deep Koopman predictor is still a lifted-space approximation of the underlying vehicle dynamics. As a result, prediction accuracy and control performance may degrade when the vehicle operates far outside the data-supported regime. Moreover, the road-boundary and obstacle terms are introduced as soft safety-shaping penalties, which improve safety awareness in the policy-learning process and enlarge empirical safety margins. However, the current framework provides a practical safety-enhancement mechanism rather than a strict safety certificate for the real system under model mismatch and disturbances. Finally, although the current experiments cover representative changes in road geometry, surface condition, and speed, the present validation is still limited in scope. Broader evaluation under more diverse operating conditions, such as larger variations in road friction, vehicle mass/loading, and higher-speed maneuvers, would further strengthen the assessment of the proposed framework.
}

\section{Conclusion and Future Work}

This paper proposed a learning predictive control (LPC) framework with deep Koopman operators for autonomous-vehicle motion planning in dynamic environments. By employing a deep Koopman vehicle model, the framework lifts nonlinear vehicle dynamics into a high-dimensional observable space with an approximately linear prediction structure, which supports efficient receding-horizon optimization. In contrast to conventional MPC, which computes open-loop control sequences online, the proposed method combines actor--critic learning with receding-horizon planning to generate a closed-loop state-feedback policy within each prediction interval. To handle non-convex environmental constraints, the framework further integrates an enhanced convex feasible set construction and potential-field-based safety shaping into the planning and policy-learning process.

\textcolor{black}{Extensive simulations in static and dynamic obstacle-avoidance scenarios on straight and curved roads demonstrate that the proposed method achieves favorable performance in safety, driving comfort, and computational efficiency relative to benchmark methods such as CBF-MPC and LMPCC. Real-vehicle experiments on the HongQi-EHS3 platform further verify the practical applicability of the proposed method, showing stable obstacle-avoidance performance in both a structured straight-road scenario and a more challenging curved mixed-surface scenario.}
\textcolor{black}{Future work will focus on adaptive tuning of safety-shaping parameters, broader validation under more challenging dynamic-traffic scenarios and road/surface conditions, and further improvement of robustness to model mismatch and operating-condition variations.}

\bibliographystyle{unsrt}
\bibliography{ref}

\appendix
\subsection{Learning Convergence in Each Prediction Interval}
\label{app:inner_loop_convergence}
\setcounter{equation}{0}
\renewcommand{\theequation}{A.\arabic{equation}}

In the following, we analyze the convergence of the policy
learning procedure in each prediction interval. For simplicity, let
\[
B_{b,\tau}=B_b(d_{b,\tau}), \qquad B_{o,\tau}=B_o(\phi(x_\tau)),
\]
then the stage cost function can be expressed as
\begin{equation}
r_{\tau}^{(i)}(\phi_{e,\tau},u_\tau)
=
\phi_{e,\tau}^{\top}Q\phi_{e,\tau}
+
\bigl(u_{\tau}^{(i)}\bigr)^{\top}Ru_{\tau}^{(i)}
+
\gamma_b B_{b,\tau}
+
\gamma_o B_{o,\tau},
\label{eq:A1}
\end{equation}
where $B_{b,\tau}$ and $B_{o,\tau}$ are defined by Eq. \eqref{equRHDHP:bound_barrier_function} and Eq.~\eqref{equRHDHP:obs_barrier_function}, respectively. 
\begin{lemma}[\cite{xu2018learning}]
  For any time step $\tau \in [k,k+N_p-1]$ in the prediction horizon, let the value function sequence $V^{(i)}$ be defined by Eq.~\eqref{equRHDHP:value_iteration} with initial value function $V^{(0)}=0$. Then, there exists an admissible control input $u_\tau$ and an upper bound $\bar{Z}_\tau$ such that
\[
0 \le V_\tau^{(i)} \le V_\tau^{*} \le \bar{Z}_\tau .
\]  
\end{lemma}

\begin{theorem}\label{theo:1}
[Convergence]
For the motion control optimization problem \eqref{equRHDHP:planning_optimization1} based on the deep Koopman vehicle dynamics model Eq.~\eqref{equDDK:DDK}, let $u^{(i)}$ and $V^{(i)}$ be defined by Eq.~\eqref{equRHDHP:value_iteration}, with initial condition $V_\tau^{(0)}=0$, $\tau \in [k,k+N_p]$. Then, the following hold:

\begin{enumerate}
    \item If $Z_\tau^{(0)}=V_\tau^{(0)}$, and $V_\tau^{(i)}$ satisfies Eq.~\eqref{equRHDHP:value_iteration}, then for any admissible control strategy $\eta_\tau^{(i)}$,
    \[
    Z_\tau^{(i+1)}
    =
    r^{(i)}(\phi_{e,\tau},\eta_\tau)
    +
    Z_{\tau+1}^{(i)},
    \qquad
    Z_\tau^{(i)} \ge V_\tau^{(i)} .
    \]

    \item The value function $V_\tau^{(i)}$ has an upper bound $\bar{V}_\tau$, such that
    \[
    0 \le V_\tau^{(i)} \le \bar{V}_\tau .
    \]

    \item As the iteration count $i \to \infty$, the value function and control strategy converge, i.e.,
    \[
    V_\tau^{(i)} \to V_\tau^{*},
    \qquad
    \lambda_\tau^{(i)} \to \lambda_\tau^{*},
    \qquad
    u_\tau^{(i)} \to u_\tau^{*}.
    \]
\end{enumerate}
\end{theorem}
\textit{Proof:}
First, we prove Claim 1). Since $\eta^{(i)}$ is any admissible control strategy, and $u^{(i)}$ minimizes the right-hand side of Eq.~\eqref{equRHDHP:policy_improvement}, under the same conditions, we have
\[
Z_\tau^{(i)} \ge V_\tau^{(i)} .
\]

Next, we prove Claim 2). Let $Z_\tau$ represent the value function corresponding to any admissible control strategy $\eta_\tau$, with
\[
Z_\tau^{(0)} = V_\tau^{(0)} = 0.
\]
The update rule for $V_\tau^{(i)}$ is given by Eq.~\eqref{equRHDHP:value_iteration}, while the update for $Z_\tau^{(i)}$ follows
\[
Z_\tau^{(i+1)}
=
r^{(i)}(\phi_\tau,\eta_\tau)+Z_{\tau+1}^{(i)}.
\]
Thereby, we have
\begin{equation}
Z_\tau^{(i+1)} - Z_\tau^{(i)}
=
Z_{\tau+1}^{(i)} - Z_{\tau+1}^{(i-1)}
=
\cdots
=
Z_{\tau+i}^{(1)} - Z_{\tau+i}^{(0)},
\label{eq:A2}
\end{equation}
yielding
\begin{equation}
Z_\tau^{(i+1)}
=
Z_{\tau+i}^{(1)} + Z_\tau^{(i)} .
\label{eq:A3}
\end{equation}

For $i \le k+N_p-1-\tau$, by expanding $Z_\tau^{(i)}$, we obtain
\begin{equation}
\begin{aligned}
Z_\tau^{(i+1)}
&=
Z_{\tau+i}^{(1)} + Z_{\tau+i-1}^{(1)} + Z_{\tau+i-2}^{(1)} + \cdots + Z_\tau^{(1)} \\
&=
\sum_{j=0}^{i} Z_{\tau+j}^{(1)}
=
\sum_{j=0}^{i} r(\phi_{\tau+j},\eta_{\tau+j}) .
\end{aligned}
\label{eq:A4}
\end{equation}

For $i > k+N_p-1-\tau$, $\tau \in [k,k+N_p]$, under the feasible strategy $\eta$, it is easy to obtain
\begin{equation}
Z_\tau^{(i+1)}
=
\sum_{j=0}^{k+N_p-1-\tau} r(\phi_{\tau+j},\eta_{\tau+j})
+
V_f(\phi_{k+N_p}) .
\label{eq:A5}
\end{equation}

Combining \eqref{eq:A4} and \eqref{eq:A5}, for any $i \in \mathbb{N}_{\ge 1}$, we have
\begin{equation}
Z_\tau^{(i+1)} \le \bar{V}_\tau .
\label{eq:A6}
\end{equation}
Based on Claim 1), for any $i \in \mathbb{N}_{\ge 1}$, we further have
\begin{equation}
V_\tau^{(i+1)}
\le
Z_\tau^{(i+1)}
\le
\bar{V}_\tau .
\label{eq:A7}
\end{equation}

Now, we prove Claim 3). Define
\[
Y_\tau^{(0)} = V_\tau^{(0)} = 0,
\]
and introduce $Y_\tau^{(i)} \le V_\tau^{(i+1)}$. When $i=0$, we have
\begin{equation}
V_\tau^{(1)} - Y_\tau^{(0)}
=
\phi_{e,\tau}^{\top}Q\phi_{e,\tau}
+
\gamma_b B_{b,\tau}
+
\gamma_o B_{o,\tau},
\label{eq:A8}
\end{equation}
such that
\[
V_\tau^{(1)} \ge Y_\tau^{(0)} .
\]

Assume that for all $\tau \in [k,k+N_p-1]$, we have
\[
V_\tau^{(i)} \ge Y_\tau^{(i-1)} .
\]
Since
\begin{equation}
\begin{aligned}
Y_\tau^{(i)}
&=
\phi_{e,\tau}^{\top}Q\phi_{e,\tau}
+
\bigl(u_\tau^{(i)}\bigr)^{\top}Ru_\tau^{(i)}
+
\gamma_b B_{b,\tau}
+
\gamma_o B_{o,\tau}
+
Y_{\tau+1}^{(i-1)}, \\
V_\tau^{(i+1)}
&=
\phi_{e,\tau}^{\top}Q\phi_{e,\tau}
+
\bigl(u_\tau^{(i)}\bigr)^{\top}Ru_\tau^{(i)}
+
\gamma_b B_{b,\tau}
+
\gamma_o B_{o,\tau}
+
V_{\tau+1}^{(i)},
\end{aligned}
\label{eq:A9}
\end{equation}
it follows that
\begin{equation}
V_\tau^{(i+1)} - Y_\tau^{(i)}
=
V_{\tau+1}^{(i)} - Y_{\tau+1}^{(i-1)}
\ge 0 .
\label{eq:A10}
\end{equation}
Then we further have
\[
V_\tau^{(i+1)} \ge Y_\tau^{(i)} .
\]

According to conclusion 1), we have
\[
V_\tau^{(i)} \le Y_\tau^{(i)},
\]
and therefore
\begin{equation}
V_\tau^{(i)} \le Y_\tau^{(i)} \le V_\tau^{(i+1)} .
\label{eq:A11}
\end{equation}
This means that
\[
V_\tau^{(i)} \le V_\tau^{(i+1)} .
\]
Combined with conclusion 2), $\{V_\tau^{(i)}\}$ is a non-decreasing and bounded sequence. Thus, when $i \to \infty$,
\[
V_\tau^{(i)} \to V_\tau^{*}.
\]
Consequently, we can show that
\[
\lambda_\tau^{(\infty)}
=
\frac{\partial V_\tau^{(\infty)}}{\partial \phi_\tau}
\to
\frac{\partial V_\tau^{*}}{\partial \phi_\tau}
=
\lambda_\tau^{*},
\qquad
u_\tau^{(i)} \to u_\tau^{*}.
\]
\hfill$\blacksquare$
\subsection{Safety Guarantee} 
{\color{black} Note that the constraint satisfaction result follows directly from Theorem~\ref{theo:1}. Specifically, as $\epsilon \to 0$, the potential function  $\mathcal{B} _o\left( \boldsymbol{\hat{\phi}} \left( x_{\tau} \right) \right)\rightarrow +\infty$ when the state approaches the obstacle boundary, i.e., $\bar{C}\boldsymbol{\hat{\phi} }\left( x_{\tau} \right) \in \partial\mathbb{F} _{\tau}$ for all $\tau\in[k,k+N_p-1]$. Meanwhile, Theorem~\ref{theo:1} ensures that the optimal value $V_{\tau}$ remains bounded. Consequently, the optimal solution $\bar{C}\boldsymbol{\hat{\phi} }\left( x_{\tau} \right) \in \mathbb{F} _{\tau}$.
However, this does not guarantee that the true state satisfies $\bar{C}\boldsymbol{\phi}(x_k) \in \mathbb{F}_k$ at every time instant $k$, due to model mismatch introduced by the Koopman approximation. One practical way to bridge this gap is to construct the potential function with appropriately chosen parameters $g_{i,k}$ to enlarge the safety margin and thus enhance robustness against modeling errors.
\begin{remark}
    Note that our approach is implemented in a receding-horizon manner, which naturally induces a closed-loop system. Within each prediction interval, constraint satisfaction is promoted via the designed potential function, ensuring that the predicted trajectory remains within the safe set. However, establishing closed-loop guarantees, such as stability or robustness for the true system, is challenging in the presence of model mismatch introduced by the learned Koopman representation and the time-varying feasible set induced by complex obstacles. In particular, such guarantees typically require strong assumptions, e.g., the existence of invariant safe sets that do not intersect with obstacles over the entire trajectory. In this case, one can readily derive the closed-loop robustness results following the line in~\cite{zhang2024model,zhang2025toward}.
    
Instead, this work focuses on ensuring safety at the planning level with an enlarged margin, which provides a practical robustness mechanism against modeling errors. The resulting closed-loop behavior is validated through extensive simulations and real-world experiments, demonstrating reliable and safe performance in complex environments.
\end{remark}
}

\end{document}